\documentclass[nohyperref]{article}

\usepackage{microtype}
\usepackage{graphicx}
\usepackage{subfigure}
\usepackage{booktabs} 

\usepackage{hyperref}

\usepackage[accepted]{icml2022}

\usepackage{amsmath}
\usepackage{amssymb}
\usepackage{mathtools}
\usepackage{amsthm}

\usepackage[capitalize,noabbrev]{cleveref}

\theoremstyle{plain}
\newtheorem{theorem}{Theorem}[section]
\newtheorem{proposition}[theorem]{Proposition}

\theoremstyle{definition}
\newtheorem{definition}[theorem]{Definition}

\theoremstyle{remark}
\newtheorem{remark}[theorem]{Remark}

\usepackage[textsize=tiny]{todonotes}

\icmltitlerunning{Label-Free Explainability for Unsupervised Models}

\usepackage{amsmath}
\usepackage{amssymb}
\usepackage{nicefrac}
\usepackage{pifont}
\usepackage{transparent}

\usepackage{makecell}
\usepackage{adjustbox}
\usepackage{multirow}

\usepackage{placeins}
\usepackage{setspace}

\newcommand{\X}{\mathcal{X}}
\newcommand{\Y}{\mathcal{Y}}
\renewcommand{\H}{\mathcal{H}}
\newcommand{\Z}{\mathcal{Z}}
\newcommand{\D}{\mathcal{D}}
\newcommand{\R}{\mathbb{R}}
\newcommand{\N}{\mathbb{N}}
\newcommand{\f}{\boldsymbol{f}}
\newcommand{\m}{\boldsymbol{m}}
\newcommand{\x}{\boldsymbol{x}}
\newcommand{\z}{\boldsymbol{z}}
\newcommand{\y}{\boldsymbol{y}}
\newcommand{\h}{\boldsymbol{h}}
\newcommand{\param}{\boldsymbol{\theta}}
\newcommand{\params}{\param_*}
\newcommand{\paramh}{\hat{\param}}
\newcommand{\hessian}[1][\params]{\boldsymbol{H}_{#1}}
\newcommand{\shift}[1][n]{\boldsymbol{\delta}^{#1}}
\newcommand{\paramshift}[1][n]{\shift[#1] \param}
\newcommand{\modelshift}[1]{\shift{#1} \f_{\paramh}}
\newcommand{\latentshift}[1]{\shift{#1} \h}
\newcommand{\grad}[1][\param]{\boldsymbol{\nabla}_{#1}}
\newcommand{\norm}[2][\H]{\left\| #2 \right\|_{#1}}
\newcommand{\iprod}[2][\H]{\left< #2 \right>_{#1}}
\newcommand{\Dtrain}{\D_{\textrm{train}}}
\newcommand{\Dtest}{\D_{\textrm{test}}}
\newcommand{\noise}{\boldsymbol{\epsilon}}
\newcommand{\E}{\mathbb{E}}
\newcommand{\rx}{\boldsymbol{X}}
\renewcommand{\rm}{\boldsymbol{M}}
\newcommand{\vaemu}{\boldsymbol{\mu}}
\newcommand{\vaesigma}{\boldsymbol{\sigma}}
\newcommand{\relevant}[1]{#1_{\textrm{r}}}
\newcommand{\irrelevant}[1]{#1_{\textrm{irr}}}
\newcommand{\parapoint}[1]{$\blacktriangleright$~\textbf{#1}}
\newcommand{\T}{\boldsymbol{T}}
\newcommand{\partder}[2]{\frac{\partial #1}{\partial #2}}
\newcommand{\cmark}{\ding{51}}%
\newcommand{\xmark}{\transparent{0.2}\ding{55}}%
\newcommand{\rX}{\boldsymbol{X}}
\newcommand{\A}{\mathcal{A}}

\DeclareMathOperator*{\argmax}{arg\,max}
\DeclareMathOperator*{\argmin}{arg\,min}
\DeclareMathOperator*{\cov}{cov}
\DeclareMathOperator*{\rank}{rank}

\begin{document}

\twocolumn[
\icmltitle{Label-Free Explainability for Unsupervised Models}



\icmlsetsymbol{equal}{*}

\begin{icmlauthorlist}
	\icmlauthor{Jonathan Crabbé}{cam}
	\icmlauthor{Mihaela van der Schaar}{cam,atu,ucla}
\end{icmlauthorlist}

\icmlaffiliation{cam}{University of Cambridge}
\icmlaffiliation{ucla}{University of California Los Angeles}
\icmlaffiliation{atu}{The Alan Turing Institute}

\icmlcorrespondingauthor{Jonathan Crabbé}{jc2133@cam.ac.uk}
\icmlcorrespondingauthor{Mihaela van der Schaar}{mv472@cam.ac.uk}

\icmlkeywords{Machine Learning, ICML}

\vskip 0.3in
]



\printAffiliationsAndNotice{} 

\begin{abstract}
Unsupervised black-box models are challenging to interpret. Indeed, most existing explainability methods require labels to select which component(s) of the black-box's output to interpret. In the absence of labels, black-box outputs often are representation vectors whose components do not correspond to any meaningful quantity. Hence, choosing which component(s) to interpret in a label-free unsupervised/self-supervised setting is an important, yet unsolved problem. To bridge this gap in the literature, we introduce two crucial extensions of post-hoc explanation techniques: (1) label-free feature importance and (2) label-free example importance that respectively highlight influential features and training examples for a black-box to construct representations at inference time. We demonstrate that our extensions can be successfully implemented as simple wrappers around many existing feature and example importance methods. We illustrate the utility of our label-free explainability paradigm through a qualitative and quantitative comparison of representation spaces learned by various autoencoders trained on distinct unsupervised tasks. 
\end{abstract}

\section{Introduction} \label{sec:introduction}
Are machine learning models ready to be deployed in high-stakes applications? Recent years have witnessed a success of deep models on nontrivial tasks such as computer vision~\cite{Voulodimos2018}, natural language processing~\cite{Young2018} and scientific discovery~\cite{Jumper2021, Davies2021}. The success of these models comes at the cost of their complexity. Deep models typically involve millions to billions operations in order to turn their input data into a prediction. Since it is not possible for a human user to analyze each of these operations, the models appear as \emph{black-boxes}. When the deployment of these models impact critical areas, such as healthcare, finance or justice, their opacity appears as a major obstruction~\cite{Lipton2016, Ching2018, Tjoa2020}.  

\textbf{Post-Hoc Explainability.} As a response to this transparency problem, the field of \emph{explainable artificial intelligence} (XAI) received an increasing interest, see \cite{Adadi2018, BarredoArrieta2020, Das2020} for reviews. In order to retain the approximation power of deep models, many \emph{post-hoc explainability} methods were developed. These methods complement the predictions of black-box models with various explanations. In this way, models can be understood through the lens of explanations regardless their complexity. We focus on two types of such methods. \textbf{(1)~Feature Importance} explanations highlight crucial features for the black-box to issue a prediction. Examples include Saliency~\cite{Simonyan2013}, Lime~\cite{Ribeiro2016}, Integrated Gradients~\cite{Sundararajan2017}, Shap~\cite{Lundberg2017}, DeepLift~\cite{Shrikumar2017} and Perturbation Masks~\cite{Fong2017,Crabbe2021Dynamask}. \textbf{(2)~Example Importance} explanations highlight crucial training examples for the black-box to issue a prediction. Examples include Influence Function~\cite{Cook1982, Koh2017}, Deep K-Nearest Neighbours~\cite{Papernot2018}, TraceIn~\cite{Pruthi2020} and SimplEx~\cite{Crabbe2021}.

\textbf{Supervised Setting.} The previous works almost exclusively focus on explaining models obtained in a \emph{supervised setting}. In this setting, the black box $\f : \X \rightarrow \Y$ connects two meaningful spaces: the feature space $\X$ and the label space $\Y$. By meaningful, we mean that  each axis of these spaces corresponds to a quantity known by the model's user. This is illustrated at the top of Figure~\ref{fig:supervised_lf} with an idealized prostate cancer risk predictor. Each axis of $\Y$ corresponds to a clear label: the probability of mortality with and without treatment. We can explain the predictions of this model with the importance of features/examples to make a prediction on each individual axis $y_1, y_2$  or for a combination of those axes (e.g in this case $y_1 - y_2$ is associated to the treatment effect). The key point is that, in the supervised setting, the user knows the meaning of the black-box output they try to interpret. This is not always the case in machine learning.

\begin{figure*}[h] 
	\vskip 0in
	\begin{center}
		\centerline{\includegraphics[width=.65\linewidth]{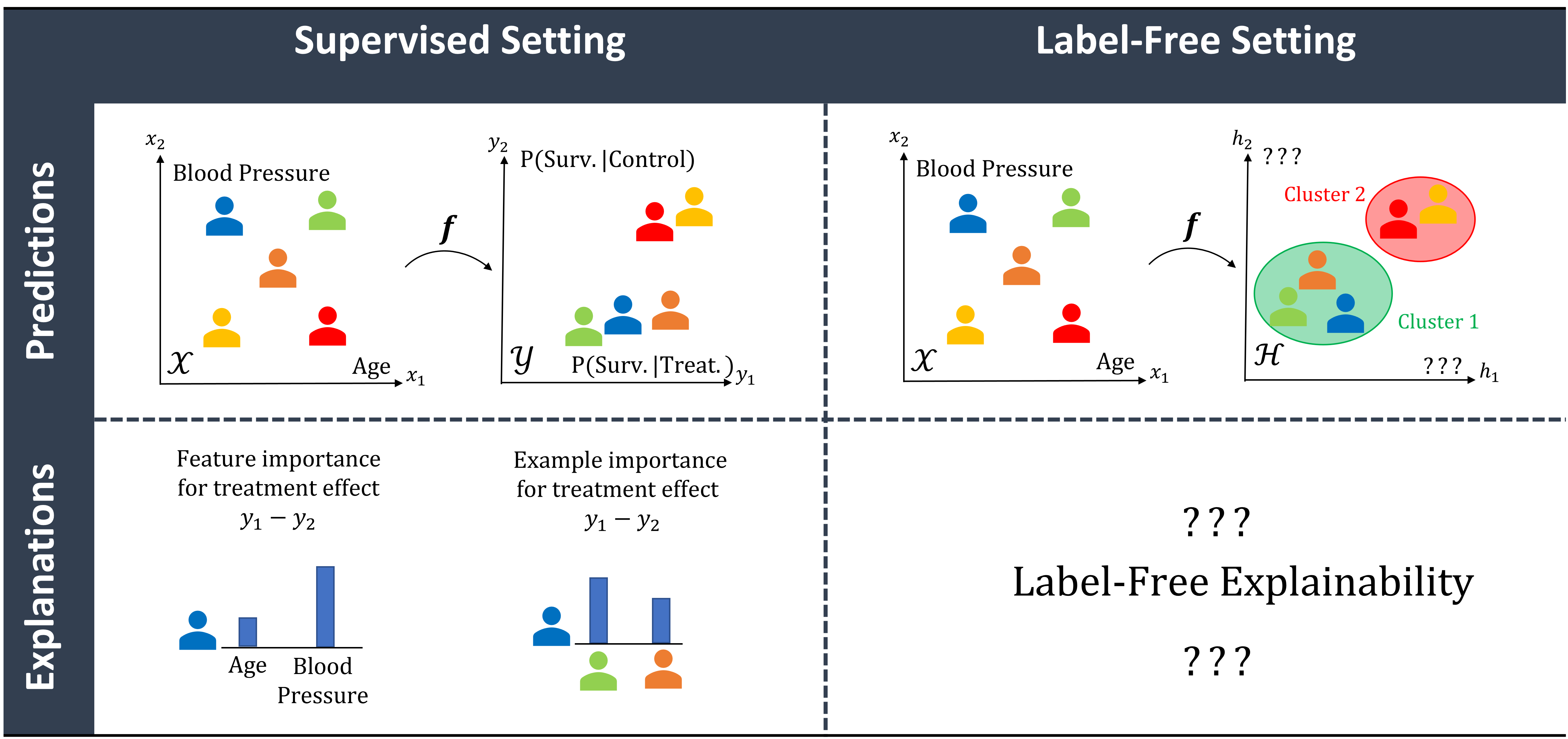}}
		\vspace{-0.1in}
		\caption{Supervised and Label-Free Settings.}
		\label{fig:supervised_lf}
	\end{center}
	\vskip -0.3in
\end{figure*}

\textbf{Label-Free Setting.} In the label-free setting, we are interested in black-boxes $\f : \X \rightarrow \H$ that connect a meaningful feature space $\X$ to a latent (or representation) space $\H$. Unlike the feature space $\X$ and the label space $\Y$, the axes of the latent space $\H$ do not correspond to labelled quantities known by the user. This is illustrated at the bottom of Figure~\ref{fig:supervised_lf}, where the examples are mapped in a representation space for clustering purposes. Unlike in the supervised setting, there is no obvious way for the user to choose an axis among $h_1, h_2$ to interpret. This distinction goes well beyond the philosophical consideration. As we show in Sections~\ref{sec:feature}~and~\ref{sec:example}, the aforementioned feature and example importance methods require labels to select an axis to interpret or evaluate a loss. The usage of these methods in the label-free setting hence requires a non-trivial extension.

\textbf{Motivation.} To illustrate the significance of the label-free extension of explainability, we outline 2 widespread setups where it is required. Note that these 2 setups are not mutually exclusive. \textbf{(1)~Unsupervised Learning:} Whenever we solve an unsupervised task such as clustering or density estimation, it is common to learn a function like $\f$ that projects the data onto a low dimensional representation space. Due to the unsupervised nature of the problem, any explanation of the latent space has to be done without label. \textbf{(2)~Self-Supervised Learning:} Even in a supervised learning setting, we have often more unlabelled than labelled data available. When this is the case, self-supervised learning preconises to leverage unlabelled data to solve an unsupervised pretext tasks such as predicting the missing part of a masked image~\cite{Jing2021}. This yields a representation $\f$ of the data that we can use as part of a model to solve the downstream task. If we want to interpret the model's representations of unlabelled examples, label-free explainability is a must. Let us now review related work.

\textbf{Related Work.} The majority of the relevant literature focuses on increasing the interpretability of representations spaces $\H$. Disentangled-VAEs constitute the best example~\cite{Higgins2017, Burgess2018, Chen2018, Sarhan2019}, we discuss them in more details in Section~\ref{subsec:vaes}. When it comes to post-hoc approaches for explainability of latent spaces, \emph{concept-based explanations}~\cite{Kim2017, Brocki2019} are central. These methods identify a dictionary between human concepts (like the presence of stripes on an image) and latent vectors. These methods are only partially relevant here since concepts are typically learned with labelled examples, although early works challenge this assumption~\cite{Ghorbani2019b}. When it comes to label-free feature importance, we note that \emph{Layer Relevance Propagation} (LRP)  permit to assign feature importance scores in the absence of labels~\cite{Bach2015, Eberle2022, Holzinger2022}. That said, these works have been restricted to specific settings (e.g. clustering and similarity models) and come with the natural limitations of LRP-based methods (e.g. no implementation invariance). Similarly, works discussing label-free example importance are also restricted to specific settings, like \citet{Kong2021} that adapt TraceIn to VAEs. 

\textbf{Contribution.}  \textbf{(1)~Label-Free Feature Importance:} We introduce a general framework to extend \emph{linear} feature importance methods to the label-free setting (Section~\ref{sec:feature}). Our extension is done by defining an auxiliary scalar function as a wrapper around the label-free black-box to interpret. This permits to compute feature importance in the label-free setting by retaining useful properties of the original methods without increasing their complexity. \textbf{(2)~Label-Free Example Importance:} We extend  example importance methods to the label-free setting (Section~\ref{sec:example}). In this work, we treat two types of example importance methods that we call \emph{loss-based} and \emph{representation-based}. For the former, the extension requires to specify a label-free loss and a set of relevant model parameters to differentiate the loss with. For the latter, the extension is straightforward. Our feature and example importance extensions are validated experimentally (Section~\ref{subsec:consistency_checks}) and their practical utility is demonstrated with a use case motivated by self-supervised learning (Section~\ref{subsec:pretext}).  \textbf{(3)~Challenging Interpretability of Disentangled Representations:} In testing the limits of our feature importance hypotheses with disentangled VAEs, we noticed that the interpretability of saliency maps associated to individual latent units seems unrelated to the strength of disentanglement between the units (Section~\ref{subsec:vaes}). We analyze this phenomenon both qualitatively and quantitatively. This insight could be the seed of future developments in interpretable representation learning. 

\section{Feature Importance} \label{sec:feature}
In this section, we present our approach to extend linear feature importance methods to the label-free setting. We start by reviewing the typical setup with label to grasp some useful insights for our extension. With these insights, the extension to the label-free regime is immediate. 

\subsection{Feature Importance with Labels}
We consider an input (or feature) space $\X \subset \R^{d_X}$ and an output (or label) space $\Y \subset \R^{d_Y}$, where $d_X \in \N^*$ and $d_Y \in \N^*$ are respectively the dimension of the input and output spaces. We are given a black-box model $\f : \X \rightarrow \Y$ from a hypothesis set\footnote{Typically neural networks.} $\A(\Y^{\X})$ mapping each input $\x \in \X$ to an output $\y = \f(\x) \in \Y$. Note that we use bold symbols to emphasize that those elements are typically vectors with more than one component ($d_X , d_Y > 1$). Feature importance methods explain the black-box prediction $\f(\x)$ by attributing an importance score $a_i(\f , \x)$ to each feature $x_i$ of $\x$ for\footnote{We denote by $[N]$ the positive integers from 1 to $N \in \N^*$.} $i \in [d_X]$. Note that feature importance methods require to select one component $f_j (\x) \in \R$ for some $j \in [d_Y]$ of the output in order to compute these scores: $a_i(\f , \x) \equiv a_i(f_j , \x)$. In a classification setting, $j$ typically corresponds to the ground-truth label (when it is known) $j = \argmax_{k \in [d_Y]} \left[y_k \right]$ or to the label with maximal predicted probability $j = \argmax_{k \in [d_Y]} \left[f_k(\x) \right]$.

We now suggest an alternative approach: combining the importance scores for each component of $\f$ by weighing them with the associated class probability: $b_i(\f , \x) \equiv \sum_{j=1}^{d_y} f_j(\x) \cdot a_i(f_j , \x)$. We note that, when a class probability $f_j(\x) \approx 1$ dominates, this reduces to the previous definition. However, when the class probabilities are balanced, this accounts for the contribution of each class. In the image classification context, this might be more appropriate than cherry-picking the saliency map of the appropriate class while disregarding the others~\cite{Rudin2019}. To the best of our knowledge, this approach has not been used in the literature. A likely reason for this is that this definition requires to compute $d_Y \cdot d_X$ importance scores per sample, which quickly becomes expensive as the number of classes grows. This limitation can easily be avoided when the importance scores are linear with respect to the black-box\footnote{Which is the case for most methods, including Shap.}. In this case, the weighted importance score can be rewritten as $b_i(\f , \x) = a_i (\sum_{j=1}^{d_Y} f_j(\x) \cdot f_j , \x )$. With this trick, we can compute the weighted importance score by only calling the auxiliary function $g_{\x}$ defined for all $\tilde{\x} \in \X$ as $g_{\x}(\tilde{\x}) = \sum_{j=1}^{d_Y} f_j(\x) \cdot f_j(\tilde{\x})$. We will now use a similar reasoning in the label-free setting.

\subsection{Label-Free Feature Importance}
We now turn our setting of interest.  We consider a latent (or representation) space $\H \subset \R^{d_H}$, where $d_H \in \N^*$ is the dimension of the latent space. We are given a black-box model $\f : \X \rightarrow \H$ from a hypothesis set $\A(\H^{\X})$ mapping each input $\x \in \X$ to a representation  $\h = \f(\x) \in \H$. As aforementioned, the latent space dimensions $h_j , j \in [d_H]$ are not related to labels with clear interpretations. Again, we would like to attribute an importance score $b_i(\f, \x)$ to each feature $x_i$ of $\x$ for $i \in [d_X]$. Ideally, this score should reflect the importance of feature $x_i$ in assigning the representation $\h = \f(\x)$. Unlike in the previous setting, we do not have a principled way of choosing a component $f_j$ for some $j \in [d_H]$. How can we compute importance scores?

We can simply mimic the approach described in the previous section. For a feature importance method, we use the weighted importance score $b_i(\f , \x) = a_i (\sum_{j=1}^{d_H} f_j(\x) \cdot f_j , \x)$. We stress that the individual components $f_j(\x)$ do not correspond to probabilities in this case. Does it still make sense to compute a sum weighted by these components? In most cases, it does. The components will typically correspond to a neuron's activation function~\cite{Glorot2011}. With typical activation functions such as ReLU or Sigmoid, inactive neurons will correspond to a vanishing component $f_j(\x) = 0$. From the above formula, this implies that these neurons will not contribute in the computation of $b_i(\f , \x)$. Neurons that are more activated, on the other hand, will contribute more to the weighted sum. By linearity of the feature importance method, this reasoning extends to linear combinations of neurons. We note that the weighted sum is a latent space inner product. This leads to the following definition.

\begin{definition}[Label-Free Feature Importance] \label{def:lf_feature_importance}
	Let $\f:\X \rightarrow \H$ be a black-box latent map and for all $i \in [d_X]$ a $a_i (\cdot , \cdot) : \A(\R^{\X}) \times \X  \rightarrow \R $ be a feature importance score linear w.r.t. its first argument. We define the \emph{label-free} feature importance as a score $b_i (\cdot , \cdot) : \A(\H^{\X}) \times \X  \rightarrow \R$:
	\begin{align} \label{eq:lf_feature_importance}
		&b_i(\f, \x) \equiv a_i \left( g_{\x} , \x \right) \\
		&g_{\x} : \X \rightarrow \R \text{  such that for all } 
		 \tilde{\x} \in \X: \nonumber\\
		 & g_{\x}(\tilde{\x}) = \iprod{\f(\x), \f(\tilde{\x})}, \label{eq:auxiliary_function}
	\end{align} 
where $\iprod{\cdot , \cdot}$ denotes an inner product for the space $\H$.
\end{definition}
\begin{remark}
	This definition gives a simple recipe to extend any linear feature importance method $a_i$ to the label-free setting. In practice, this is implemented by defining the auxiliary \emph{scalar} function $g_{\x}$ as a simple wrapper around the black-box function $\f$. We then feed $g_{\x}$ to any feature importance method $a_i$.
\end{remark}

Arguably one of the most important property shared by many feature importance methods is \emph{completeness}. Feature importance methods endowed with this property produce importance scores whose sum equals the black-box prediction up to a constant baseline $a_0 \in \R$: $\sum_{i=1}^{d_x} a_i (g, \x) = g(\x) - a_0$. This provides a meaningful connection between importance scores and black-box predictions. Typical examples of baselines are $a_0=0$ for Lime; the expected prediction $a_0 = \mathbb{E}_{\boldsymbol{X}} \left[g(\boldsymbol{X})\right]$ for Shap and a baseline prediction $a_0 = g(\bar{\x})$ for Integrated Gradients.  We show that our label-free feature importance scores are endowed with an analogous property in higher dimension. 

\begin{proposition}[Label-Free Completeness] \label{proposition:completeness}
	If the feature importance scores $a_i (\cdot , \cdot) : \A(\R^{\X}) \times \X  \rightarrow \R $ are linear and satisfy completeness, then the label-free importance scores $b_i(\f , \x) , i \in [d_X]$ sum to the black-box representation's norm $\norm{f(\x)}^2 = \iprod{\f(\x), \f(\x)} $ up to a constant baseline $b_0 \in \R$ for all $\x \in \X$:
	\begin{align} \label{eq:lf_completeness}
		\sum_{i=1}^{d_X} b_i(\f , \x) = \norm{\f(\x)}^2 - b_0.
	\end{align} 
\end{proposition}
\begin{proof}
	The proof is provided in Appendix~\ref{appendix:properties}.
\end{proof}

This property is more general than its LRP counterpart (Proposition 1 in \cite{Eberle2022}) that holds only for neural network with biases set to zero. Furthermore, in Appendix~\ref{subappendix:invariance}, we demonstrate that our label-free extension of feature importance verifies crucial invariance properties with respect to latent space symmetries.

\section{Example Importance} \label{sec:example}
In this section, we present our approach to extend example importance methods to the label-free setting. Since example importance methods are harder to unify, the structure of this section differs from the previous one. We split the example importance methods in two families: the loss-based and representation-based methods. The extension to the label-free setting works differently for these two families. Hence, we treat them in two distinct subsections. In both cases, we work with an input space $\X$ and a latent space $\H$. We are given a training set of $N \in \N^*$ examples $\Dtrain = \left\{ \x^n   \mid n \in [N] \right\}$. This training set is used to fit a black-box latent map $\f : \X \rightarrow \H$. We want to assign an example importance score $c^n(\f, \x)$ to each training example $\x^n \in \Dtrain$ for the black-box $\f$ to build a representation $\f(\x) \in \H$ of a test example $\x \in \X$. Note that we use upper indices for examples in contrast with the lower indices for the features. Hence, $x^n_i$ denotes feature $i$ of training example $n$. Similarly, $c^n$ denotes an example importance in contrast with $b_i$ that is used for feature importance.

\subsection{Loss-Based Example Importance} \label{subsec:loss_based}
\textbf{Supervised Setting.} Loss-based example importance methods assign a score to each training example $\x^n$ by simulating the effect of their removal on the loss for a test example. To make this more explicit, we denote by $\z$ the data of an example that is required to evaluate the loss. In a supervised setting, this typically correspond to a couple $\z = (\x , \y)$ with an input $\x \in \X$ and a label $\y \in \Y$. Similarly, the training set is of the form $\Dtrain = \left\{ \z^n   \mid n \in [N] \right\}$. This training set is used to fit a black-box model $\f_{\theta}$ parametrized by $\param \in \Theta \subset \R^P$, where $P \in \N^*$ is the number of model parameters. Training is done by solving $\params = \argmin_{\param \in \Theta} \sum_{n=1}^N  L(\z^n, \param)$ with the loss $L: \Z \times \Theta \rightarrow \R$. This yields a model $\f_{\params}$. If we remove an example $\z^n$ from $\Dtrain$, the optimization problem turns into $\params^{-n} = \argmin_{\param \in \Theta} \sum_{m=1, m \neq n}^N  L(\z^m, \param)$. This creates a parameter shift $\paramshift \equiv \params^{-n}  - \params$. This parameter shift, in turns, impacts the loss $L(\z, \param)$ on a test example $\z$. This shift is reflected by the quantity $\shift_{\param} L(\z, \params) \equiv L(\z, \params^{-n}) - L(\z , \params)$. This loss shift permits to draw a distinction between proponents ( $\z^n$ that decrease the loss: $\shift_{\param} L(\z, \params) > 0$) and opponents ($\z^n$ that increase the loss: $\shift_{\param} L(\z, \params) < 0$). Hence, it provides a meaningful measure of example importance. In order to estimate the loss shift without retraining the model, \citet{Koh2017} propose to evaluate the influence function:
\begin{align*}
	\shift_{\param} L(\z, \params) \approx \frac{1}{N} \iprod[\Theta]{\grad L(\z, \params) ,  \hessian^{-1} \grad L(\z^n, \params) }, 
\end{align*}
where $\iprod[\Theta]{\cdot, \cdot}$ is an inner product on the parameter space $\Theta$ and $\hessian \equiv \sum_{n=1}^N N^{-1} \grad^2 L(\z^n, \params) \in \R^{P \times P}$ is the training loss Hessian matrix. Similarly, \citet{Pruthi2020} propose to use checkpoints during the model training to evaluate the loss shift:
\begin{align*}
	\shift_{\param} L(\z, \params) \approx \sum_{t=1}^T \eta_t \iprod[\Theta]{\grad L(\z, \param_t) ,  \grad L(\z^n, \param_t) }, 
\end{align*}
where $\eta_t$ and $\param_t$ are respectively the learning rate and the model's parameters at checkpoint $t \in [T]$, $T \in \N^*$ is the total number of checkpoints. Similar approaches building on the theory of Shapley values exist~\cite{Ghorbani2019, Ghorbani2020}.

\textbf{Label-free Setting.} We now turn to the label-free setting. In this case, we train our model with a label-free loss $L: \X \times \Theta \rightarrow \R$. Is it enough to drop the label and fix $\z = \x$ in all the above expressions? Most of the time, no. It is important to notice that the latent map $\f : \X \rightarrow \H$ that we wish to interpret is not necessarily equal to the model that we use to evaluate the loss $L(\z , \param)$. To understand, it is instructive to consider a concrete example. Let us assume that we are in a self-supervised setting and that we train an autoencoder $\f_d \circ \f_e: \X \rightarrow \X$ that consists in an encoder $\f_e : \X \rightarrow \H$ and a decoder $\f_d : \H \rightarrow \X$ on a pretext task such as denoising. At the end of the pretraining, we typically keep the encoder $\f_e$ and drop the decoder $\f_d$. If we want to compute the loss-based example importance, we are facing a difficulty. On the one hand, we would like to draw interpretations that rely solely on the encoder $\f_e$ that is going to be used in a downstream task. On the other hand, loss-based example importance scores are computed with the autoencoder $\f_d \circ \f_e$ that also involves the irrelevant decoder $\f_d$. Whenever the loss evaluation involves more than the black-box latent map $\f$ we wish to interpret, replacing $\z$ by $\x$ in the above expressions is therefore insufficient.

To provide a more satisfactory solution, we split the parameter space $\Theta = \relevant{\Theta} \times \irrelevant{\Theta}$ into a relevant component $\relevant{\Theta} \subseteq \R^{\relevant{P}}$ and an irrelevant component $\irrelevant{\Theta} \subseteq \R^{\irrelevant{P}}$. The black-box $\f$ to interpret is parametrized only by the relevant parameters $\relevant{\param}$ and can be denoted $\f_{\relevant{\param}}$. It typically corresponds to an encoder, as in the previous example. Concretely, we are interested in isolating the part of loss shift $	\shift_{\param} L(\z, \params)$ caused by the variation of these relevant parameters. We note that the above estimators for this quantity involve gradients with respect to the parameters $\param$. We decompose the overall parameters gradient in terms of relevant and irrelevant parameters gradients: $\grad[\param] = \grad[\relevant{\param}] + \grad[\irrelevant{\param}]$. Ignoring the variation of irrelevant parameters is equivalent to setting $\grad[\irrelevant{\param}] = 0$ in the above expressions. This is trivially equivalent to the replacement of $\shift_{\param}$ by $\shift_{\relevant{\param}}$. This motivates the following definition.  

\begin{definition}[Label-Free Loss-Based Importance]
	Let $f_{\relevant{\param}}:\X \rightarrow \H$ be a black-box latent map trained to minimize a loss $L: \X \times \Theta \rightarrow \R$ on a training set $\Dtrain = \left\{ \x^n   \mid n \in [N] \right\}$. To measure the impact of removing example $\x^n$ from $\Dtrain$ with $n \in [N]$, we define the \emph{Label-Free Loss-Based Example Importance} as a score $c^n(\cdot, \cdot): \A(\H^{\X}) \times \X \rightarrow \R$ such that
	\begin{align}
		c^n(\f_{\relevant{\param}}, \x) = \shift_{\relevant{\param}} L(\x, \param_*).
	\end{align} 
\end{definition}
\begin{remark}
	This definition gives a simple recipe to extend any loss-based example importance method to the label-free setting. In practice, this is implemented by using the unsupervised loss $L$ trained to fit the model and differentiating $L$ with respect to parameters of the encoder we wish to interpret.  
\end{remark}

We stress that the loss depends only on a \emph{single} input example $\x$. This is obviously not the case for \emph{contrastive losses} that involve \emph{pairs} of input examples~\cite{Chen2020}. There is not obvious extension of loss-based example importance in this setting. We will now present another type of example importance method that better extends to contrastive learning. 

\subsection{Representation-Based Example Importance} \label{subsec:representation_based}

\textbf{Supervised Setting.} Although representation-based example importance methods are introduced in a supervised context, their extension to the label-free setting is straightforward. These methods assign a score to each training example $\x^n$ by analyzing the latent representations of these examples. To make this more concrete, we start with a typical supervised setting. Consider a model $\f_l \circ \f_e : \X \rightarrow \Y$, where $\f_e : \X \rightarrow \H$ maps inputs to latent representations and $\f_l : \H \rightarrow \Y$ maps representations to labels. In the case of deep neural networks, the representation space typically corresponds to the output of an intermediate layer. We would like to see how the representation map relates a test example $\x \in \X$ to the training set examples. This can be done by mapping the training set inputs into the representation space $\f_e (\Dtrain) = \{ \f_e(\x^n) \mid n \in [N] \}$. To quantify the affinity between $\x$ and the training set examples, we attempt a reconstruction of $\f_e (\x)$ with training representations from $\f_e(\Dtrain)$: $\f_e(\x) \approx \sum_{n=1}^N w^n(\x) \cdot \f_e(\x^n)$. Following~\cite{Papernot2018}, the most obvious approach to define weights $w^n(\x)$ is to identify the indices $\textrm{KNN}(\x) \subset [N]$ of the $K$ nearest neighbours (DKNN) of $\f_e(\x)$ in $\f_e(\Dtrain)$ and weigh them according to a Kernel function $\kappa: \H^2 \rightarrow \R^+$:
\begin{align} \label{eq:dknn}
	w^n(\x) = \boldsymbol{1}\left[n \in \textrm{KNN}(\x)\right] \cdot \kappa \left[ \f_e(\x^n) , \f_e(\x) \right], 
\end{align}    
where $\boldsymbol{1}$ denotes the indicator function. Similarly, \citet{Crabbe2021} propose to learn the weights by solving
\begin{align} \label{eq:simplex}
	\boldsymbol{w}(\x) = \argmin_{\boldsymbol{\lambda} \in [0,1]^N} \norm{\f_e(\x) - \sum_{n=1}^N \lambda^n \f_e(\x^n)} 
\end{align}
such that $\sum_{n=1}^N \lambda^n = 1$. Similar approaches building on the representer theorem exist~\cite{Yeh2018}.

\textbf{Label-free Setting.} We now turn to the label-free setting. The above discussion remains valid if we replace the supervised representation map $\f_e$ by an unsupervised representation map $\f$. In short, we can take $c^n = w^n$ without any additional work for representation-based methods. A major advantage of representation-based methods over loss-based methods is that they only require latent representations. Therefore, they naturally extend to representation spaces learned in contrastive learning. Moreover, we argue in Appendix~\ref{appendix:properties} that only representation-based methods are invariant to latent space symmetries.

\section{Experiments} \label{sec:experiments}
In this section, we conduct quantitative evaluations of the label-free extensions of various explanation methods. We start with simple consistency checks to ensure that these methods provide sensible explanations for unsupervised models. Then, we demonstrate how our label-free explanation paradigm makes it possible to compare representations learned from different pretext tasks. Finally, we challenge Definition~\ref{def:lf_feature_importance} by studying saliency maps of VAEs. A more detailed description of each experiment can be found in Appendix~\ref{appendix:experiments_details}. The implementation is available online\footnote{\url{https://github.com/JonathanCrabbe/Label-Free-XAI}} \footnote{\url{https://github.com/vanderschaarlab/Label-Free-XAI}}.

\subsection{Consistency Checks} \label{subsec:consistency_checks}

\begin{figure*}[h]
	\vspace{-0.1in}
	\centering
	\subfigure[MNIST]{
		\begin{minipage}[t]{.32\linewidth}
			\centering
			\includegraphics[width=1\linewidth]{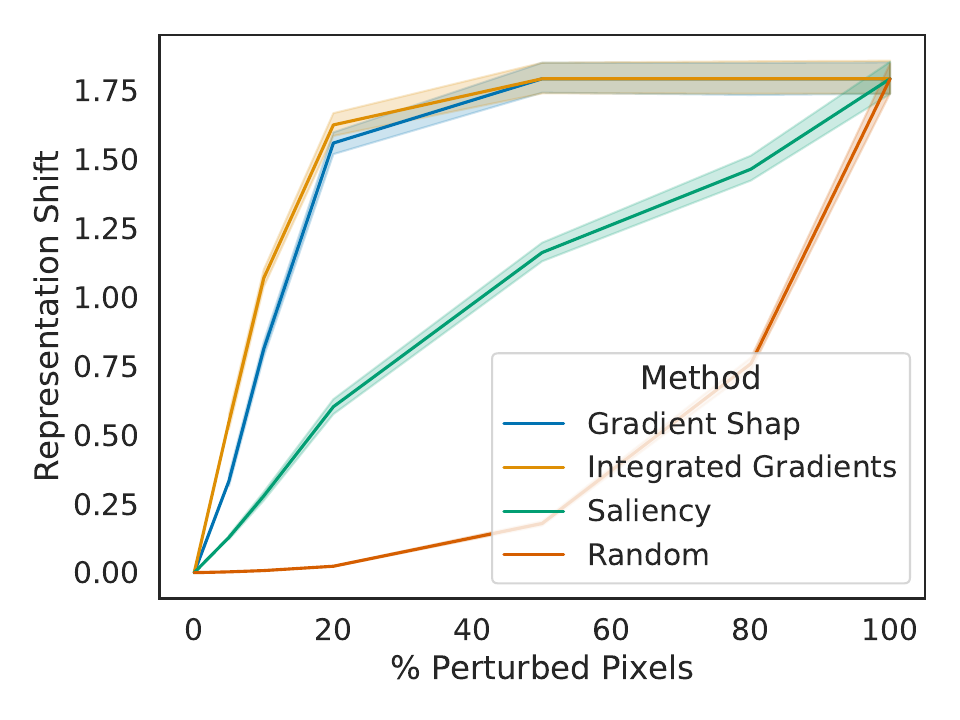}
		\end{minipage}%
	}%
	\subfigure[ECG5000]{
		\begin{minipage}[t]{.32\linewidth}
			\centering
			\includegraphics[width=1\linewidth]{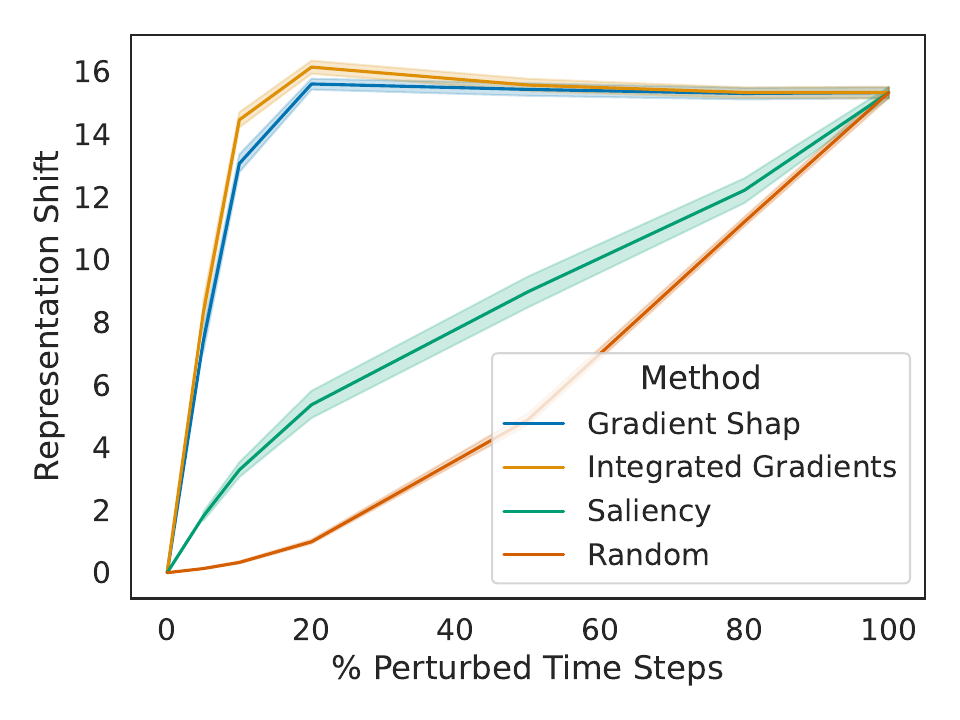}
		\end{minipage}%
	}%
	\subfigure[CIFAR-10]{
		\begin{minipage}[t]{.32\linewidth}
			\centering
			\includegraphics[width=1\linewidth]{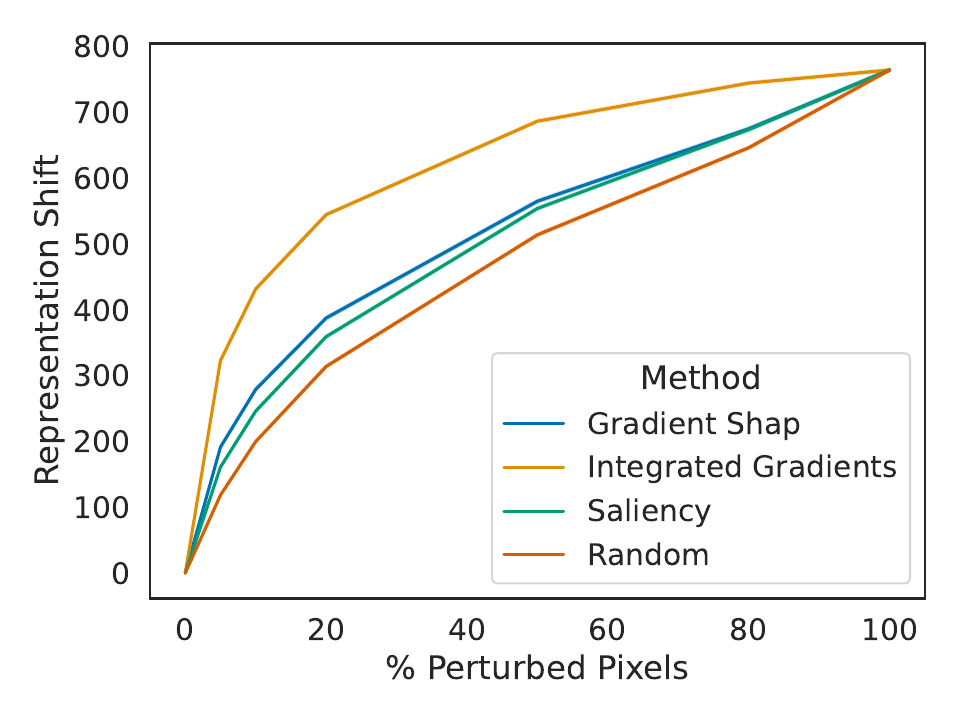}
		\end{minipage}%
	}%
	\vspace{-0.1in}
	\caption{Consistency check for label-free feature importance (average and $95 \%$ confidence interval).}
	\label{fig:cons_features}
	\vskip -0.2in
\end{figure*}

\begin{figure*}[h]
	\vspace{-0.05in}
	\centering
	\subfigure[MNIST]{
		\begin{minipage}[t]{.32\linewidth}
			\centering
			\includegraphics[width=1\linewidth]{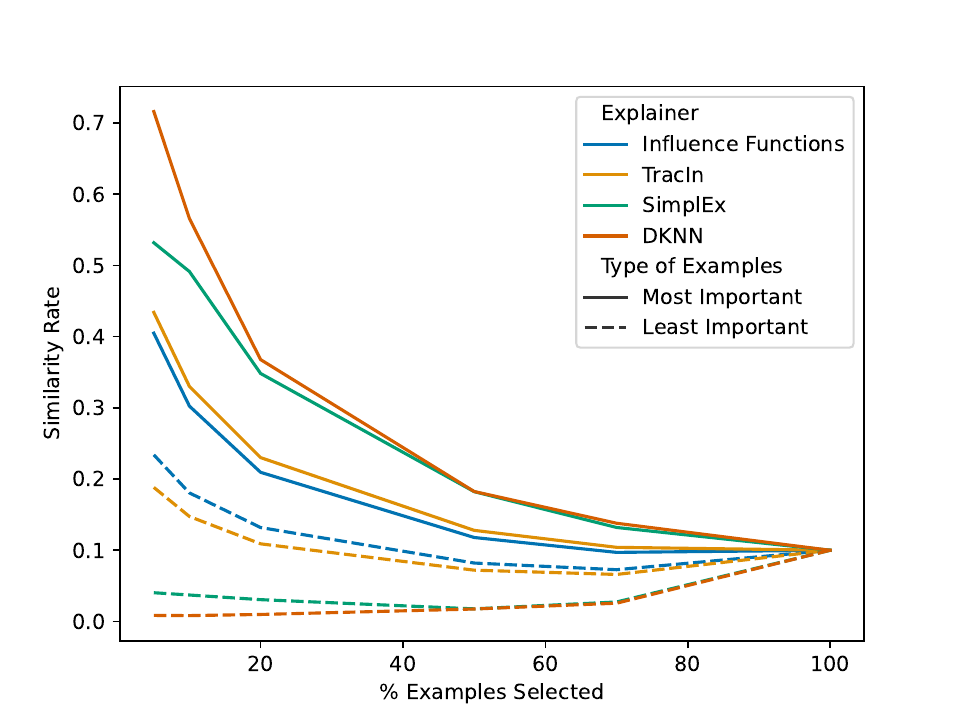}
		\end{minipage}%
	}%
	\subfigure[ECG5000]{
		\begin{minipage}[t]{.32\linewidth}
			\centering
			\includegraphics[width=1\linewidth]{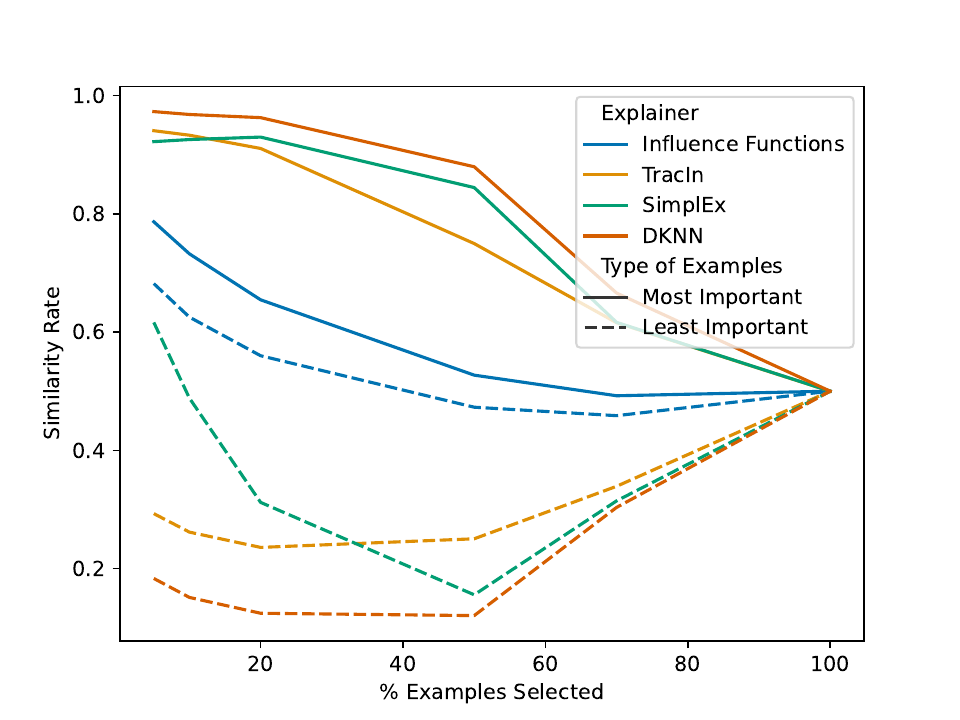}
		\end{minipage}%
	}%
	\subfigure[CIFAR-10]{
		\begin{minipage}[t]{.32\linewidth}
			\centering
			\includegraphics[width=1\linewidth]{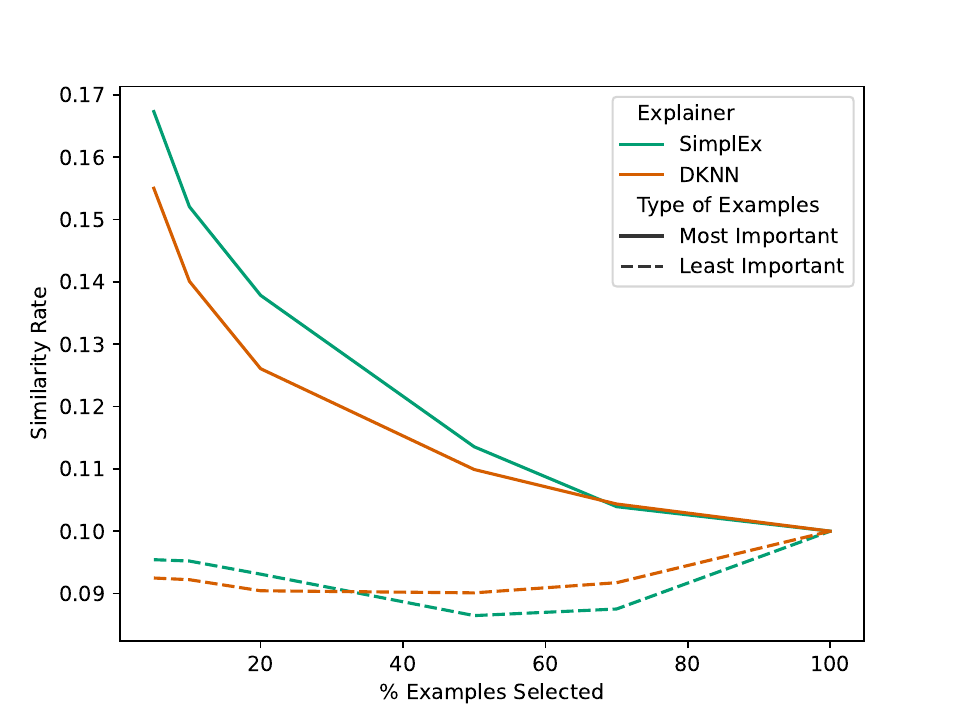}
		\end{minipage}%
	}%
	\vspace{-0.1in}
	\caption{Consistency check for label-free example importance (only representation-based methods apply to SimCLR).}
	\label{fig:cons_examples}
	\vspace{-0.1in}
\end{figure*}

We would like to asses whether the approaches described in Sections~\ref{sec:feature}~and~\ref{sec:example} provide a sensible way to extend feature and example importance scores to the unsupervised setting. 

\textbf{Setup.}  We fit 3 models on 3 datasets: a denoising autoencoder CNN on the MNIST image dataset~\cite{LeCun1998}, a LSTM reconstruction autoencoder on the ECG5000 time series dataset~\cite{Goldberger2000} and a SimCLR~\cite{Chen2020} neural network with a ResNet-18~\cite{He2015} backbone on the CIFAR-10 image dataset~\cite{Krizhevsky2009}. We extract an encoder $\f_e$ to interpret from each model. \parapoint{Feature Importance:} We compute the label-free feature importance $b_i(\f_e, \x)$ of each feature (pixel/time step) $x_i$ for building the latent representation of the test example $\x \in \Dtest$. To verify that high-scoring features are salient, we use an approach analogous to pixel-flipping~\cite{Montavon2018}: we mask the $M$ most important features with a mask $\m \in \{0,1\}^{d_X}$. We measure the latent shift $\norm{\f_e(\x) - \f_e(\m \odot \x + (1 - \m) \odot \bar{\x})}$ induced by replacing the most important features with a baseline $\bar{\x}$, where $\odot$ denotes the Hadamard product. We expect this shift to increase with the importance of masked features. We report the average shift over the testing set for several values of $M$ and feature importance methods in Figure~\ref{fig:cons_features}. \parapoint{Example Importance:} We sample $N=1000$ training examples $\x^n \in \Dtrain, n \in [N]$ without replacement. We compute the importance score $c^n(\f_e, \x)$ of each training example $\x^n$ for predicting the latent representation of the test images $\x \in \Dtest$. To verify that high scoring examples are salient, we use an approach analogous to~\citet{Kong2021}. We select the $M$ most important training examples $\x^{n_1}, \dots , \x^{n_M}$. We compare their ground truth label $y^{n_1}, \dots , y^{n_M}$ to the label $y$ of $\x$.  We compute the similarity rates $\nicefrac{\sum_{m=1}^M \delta_{y,y^{n_m}}}{M}$, where $\delta$ denotes the Kronecker delta. We reproduce the above steps for the $M$ least important examples. If the encoder meaningfully represents the data, we expect the similarity rate of the most important examples to be higher than for the least important examples. We report the distribution of similarity rates across 1,000 test examples for various values of $M$ and example importance methods in Figure~\ref{fig:cons_examples}.

\textbf{Results.} \parapoint{Feature Importance:} Label-free feature importance methods exhibit a similar behaviour: the latent shift increases sharply as we perturb the few most important pixels. This increase decelerates when we start perturbing pixels that are less relevant. Furthermore, selecting the perturbed pixels according to the various importance scores $b_i(\f_e, \x)$ yields latent shifts that are significantly larger than the shift induced by perturbing random pixels. Label-free Integrated Gradients outperform other methods for each model. These observations confirm that the label-free importance scores $b_i(\f_e, \x)$ allow us to identify the features that are the most relevant for the encoder $\f_e$ to build a latent representation for the example $\x$. \parapoint{Example Importance:} For all example importance method, we observe that the similarity rate is substantially higher among the most similar examples than among the least similar examples. This observation confirms that the label-free importance scores $c^n(\f_e, \x)$ allow us to identify training examples that are related to the test example we wish to explain. Representation-based methods usually outperform loss-based methods. In this case, the verification also validates our models since no label was used during training.

\subsection{Use Case: Comparing the Representations Learned with Different Pretext Tasks} \label{subsec:pretext}

\begin{figure*}[h]
	\vspace{-0.05in}
	\centering
	\subfigure[Top Examples]{
		\begin{minipage}[t]{.37\linewidth}
			\centering
			\includegraphics[width=1\linewidth]{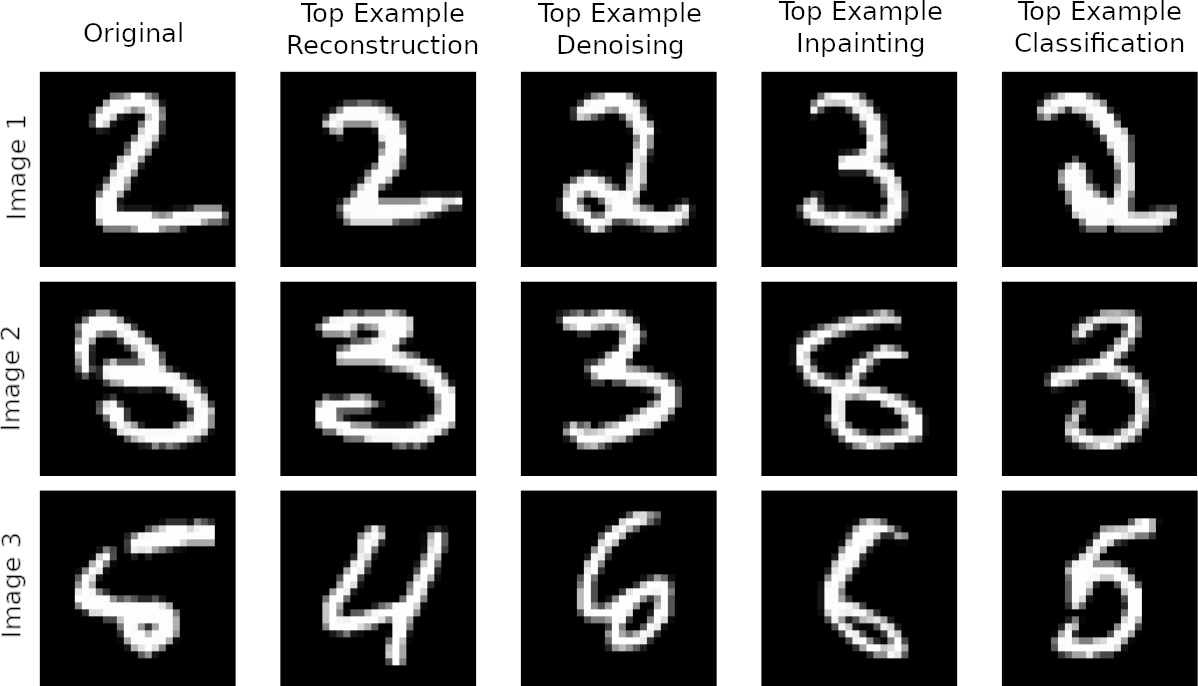}
		\end{minipage}%
	}%
	\hspace{5px}
	\subfigure[Saliency Maps]{
		\begin{minipage}[t]{.28\linewidth}
			\centering
			\includegraphics[width=1\linewidth]{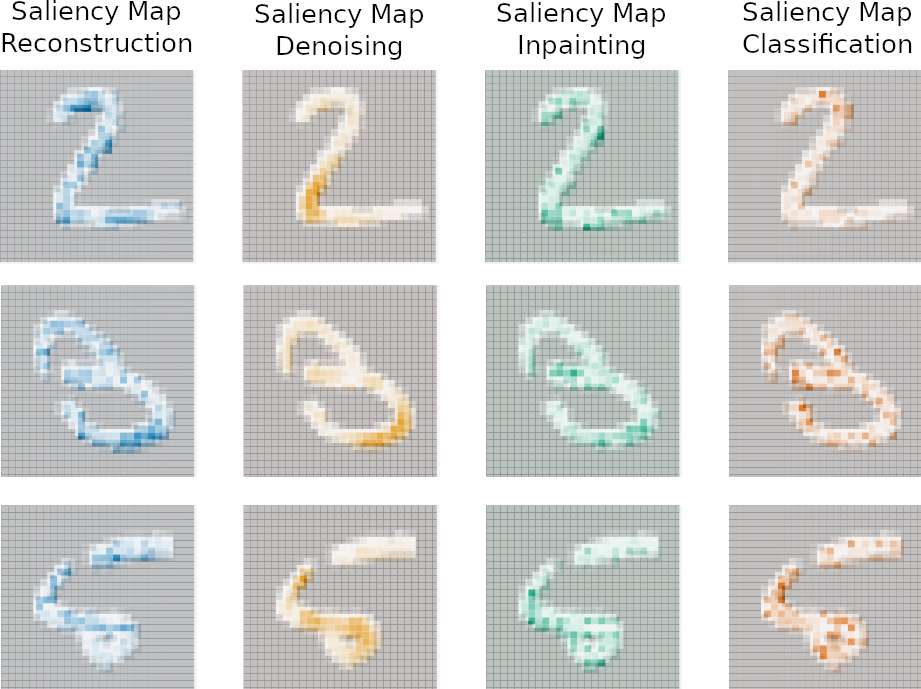}
		\end{minipage}%
	}%
	\vspace{-0.1in}
	\caption{Explanations for different pretext tasks.}
	\label{fig:examples_pretext}
	\vspace{-0.1in}
\end{figure*}

In a self-supervised learning setting, many unsupervised pretext task can be chosen to learn a meaningful representation of the data. How do the representations from different pretext tasks compare to each other? We show that label-free explainability permits to answer this question through qualitative and quantitative comparisons of representations.

\textbf{Setup.} We work with the MNIST denoising autoencoder from Section~\ref{subsec:consistency_checks}. Besides denoising, we consider 2 additional pretext tasks along with their autoencoders: reconstruction and inpainting~\cite{Pandit2019}.  Finally, we use the labelled training set to fit a MNIST classifier and extract the representations from the penultimate layer. We are interested in comparing the representation spaces learned by the encoder $\f_e$ for the various tasks. \parapoint{Feature Importance:} For each encoder $\f_e$, we use our label-free  Gradient Shap to produce saliency maps $b_i(\f_e , \x)$ for the test images $\x \in \Dtest$.  To compare the saliency maps obtained by different models, a common approach is to compute their Pearson correlation coefficient~\cite{LeMeur2013}. We report the average Pearson coefficients between the encoders across 5 runs in Table~\ref{tab:pretext_features_pearson}. \parapoint{Example Importance:} For each encoder $\f_e$, we use our label-free  Deep-KNN to produce example importance $c^n(\f_e , \x)$ of 1,000 training examples $\x^n \in \Dtrain$ for 1,000 test images $\x \in \Dtest$. Again, we use the Pearson correlation coefficient to compare different encoders. We report the average Pearson coefficients between the encoders across 5 runs in Table~\ref{tab:pretext_examples_pearson}.

\textbf{Results.} \parapoint{Not all representations are created equal.} For saliency maps, the Pearson correlation coefficients range from $.31$ to $.44$. This corresponds to moderate positive correlations. A good baseline to interpret these results is provided by ~\citet{Ouerhani2003}: the correlation between the fixation of two human subjects (human saliency maps) are typically in the same range. Hence, two encoders trained on distinct pretext tasks pay attention to different parts of the image like two separate human subjects typically do.
For example importance scores, the Pearson correlation coefficients range from $.06$ to $.12$, which correspond to weak correlations. For both explanation types, these quantitative results strongly suggest that distinct pretext tasks do not yield interchangeable representations. \parapoint{What makes classification special?} For saliency maps, the autoencoder-classifier correlations are comparable to those of the autoencoder-autoencoder. This shows that using labels creates a shift in the saliency maps comparable to changing the unsupervised pretext task. Hence, classification does not appear as a privileged task in terms of feature importance. Things are different for example importance: the autoencoder-classifier correlations are substantially lower than those of the autoencoder-autoencoder. One likely reason is that the classifier groups examples together according to an external label that is unknown to the autoencoders. 

\begin{table}[h]
	\vskip -0.1in
	\caption{Pearson correlation for saliency maps (avg +/- std).}
	\vspace{-0.1in}
	\label{tab:pretext_features_pearson}
	\begin{center}
		\begin{adjustbox}{width=\columnwidth}
			\begin{small}
				\begin{sc}
					\begin{tabular}{l|c c c c} 
						\toprule
						Pearson & Recon. & Denois. & Inpaint. & Classif. \\ 
						\hline
						Recon. &  &  &  &  \\
						Denois. & $.39 \pm .01 $ &  & &  \\
						Inpaint. & $.31 \pm .02 $ & $.32 \pm .01$ &  &  \\
						Classif. & $.44 \pm .02$ & $.40 \pm .00$ & $.32 \pm .02$ &  \\
						\bottomrule
					\end{tabular}
				\end{sc}
			\end{small}	
		\end{adjustbox}
	\end{center}
	\vskip -0.1in
\end{table}

\begin{table}[h]
	\vskip -0.1in
	\caption{Pearson correlation for example importance (avg +/- std).}
	\vspace{-0.1in}
	\label{tab:pretext_examples_pearson}
	\vskip 0.05in
	\begin{center}
		\begin{adjustbox}{width=\columnwidth}
			\begin{small}
				\begin{sc}
					\begin{tabular}{l|c c c c} 
						\toprule
						Pearson & Recon. & Denois. & Inpaint. & Classif. \\ 
						\hline
						Recon. &  &  &  &  \\
						Denois.     & $.10 \pm .04$   &  &  &   \\
						Inpaint.     & $.11 \pm .03$  & $.12 \pm .03$ &  &  \\
						Classif. & $.07 \pm .03$  & $.06 \pm .02$ & $.07 \pm .01$ &   \\
						\bottomrule
					\end{tabular}
				\end{sc}
			\end{small}	
		\end{adjustbox}
	\end{center}
\vspace{0in}
\end{table}

\textbf{Qualitative Analysis.} Beyond quantitative analysis, label-free explainability allows us to appreciate qualitative differences between different encoders. To illustrate, we plot the most important DKNN and the saliency maps for the various encoders in Figure~\ref{fig:examples_pretext}.  \parapoint{Feature Importance:} In accordance with our quantitative analysis, the saliency maps between different tasks look different. For instance, the denoising encoder seems to focus on small contiguous parts of the images. In contrast, the classifier seems to focus on a few isolated pixels. \parapoint{Example Importance:} The top examples are rarely similar across various pretext tasks, as suggested by the quantitative analysis. The classifier is the only one that associates an example of the same class given an ambiguous image like Image 3. \parapoint{Synergies:} Sometimes, saliency maps permit to better understand example importance. For instance, let us consider Image 3. In comparison to the other encoders, the reconstruction encoder pays less attention to the crucial loop at the image bottom. Hence, it is not surprising that the corresponding top example is less relevant than those selected by the other encoders.

\subsection{Challenging our Assumptions with Disentangled VAEs} \label{subsec:vaes}

\begin{figure*}[h]
	\vspace{-0.1in}
	\centering
	\subfigure[MNIST]{
		\begin{minipage}[t]{.26\linewidth}
			\centering
			\includegraphics[width=1\linewidth]{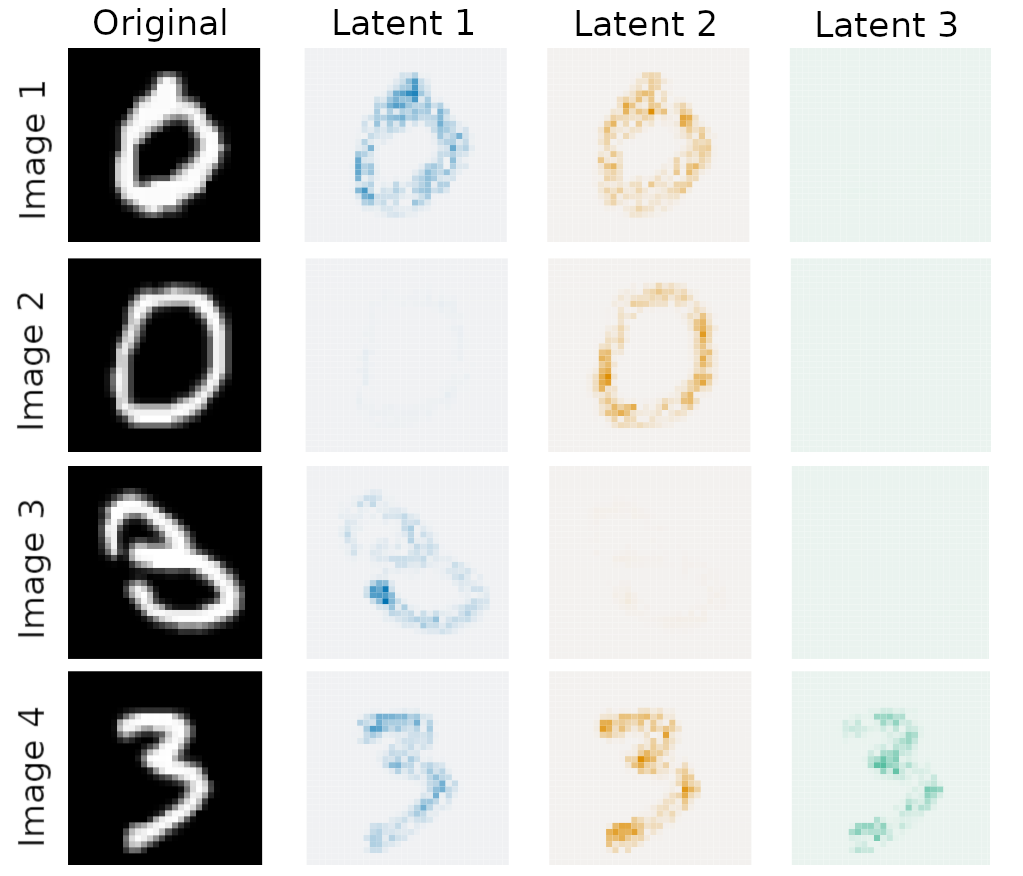}
		\end{minipage}%
	}%
	\hspace{2px}
	\rule{1px}{110px}
	\subfigure[dSprites]{ \label{subfig:examples_vae_dsprites}
		\begin{minipage}[t]{.44\linewidth} 
			\centering
			\includegraphics[width=1\linewidth]{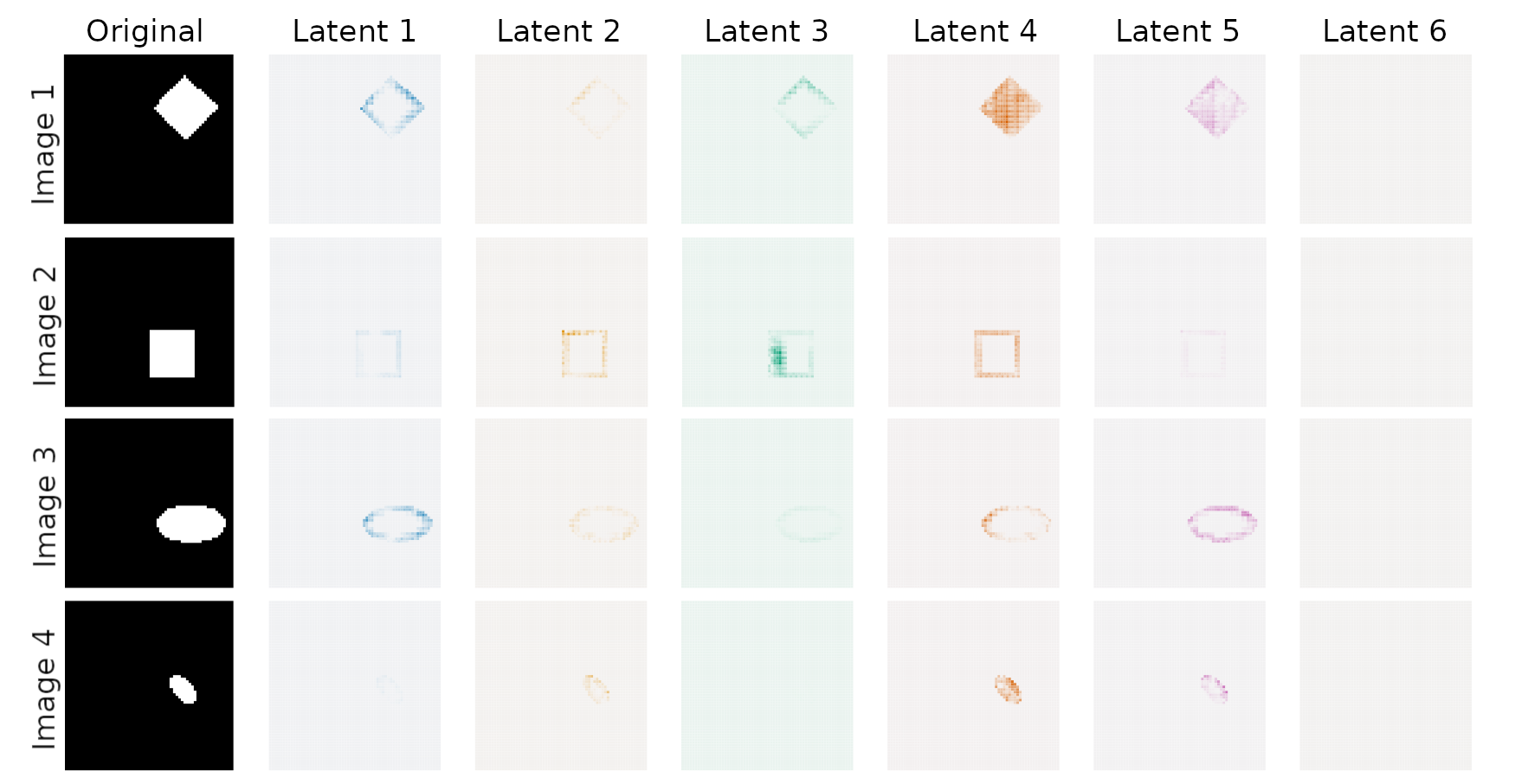}
		\end{minipage}%
	}%
	\vspace{-0.1in}
	\caption{Saliency maps for each unit of the disentangled VAEs. The scale is constant for each image.}
	\label{fig:examples_vae}
\end{figure*}

\begin{figure*}[h] 
	\vspace{-0.1in}
	\centering
	\subfigure[MNIST]{
		\begin{minipage}[t]{0.35\linewidth}
			\centering
			\includegraphics[width=\linewidth]{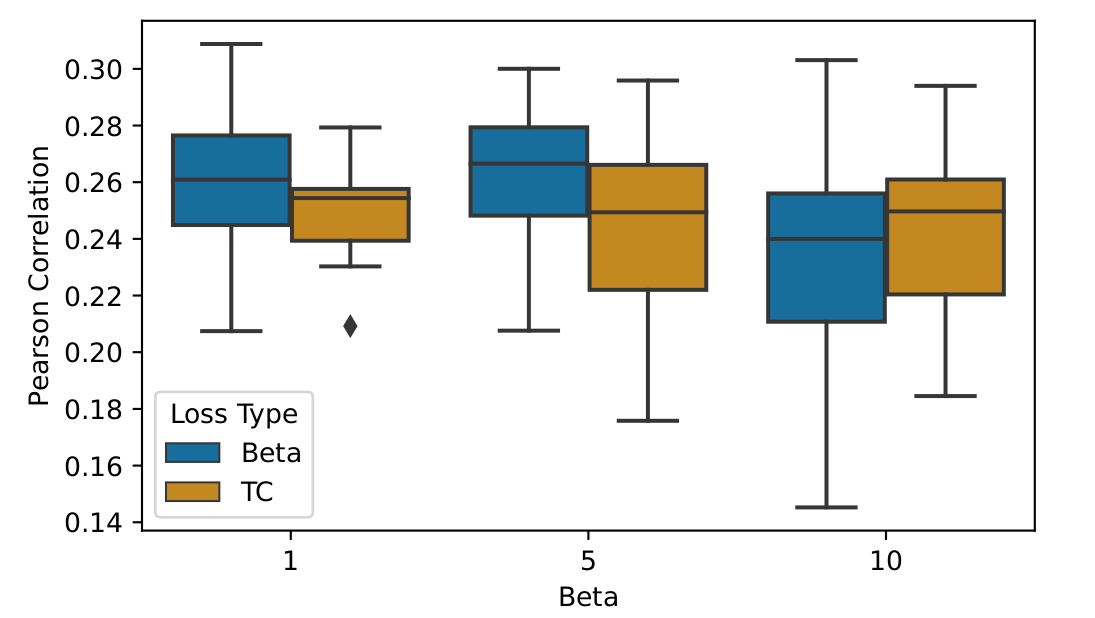}
		\end{minipage}
	}%
	\subfigure[dSprites]{
		\begin{minipage}[t]{0.35\linewidth}
			\centering
			\includegraphics[width=\linewidth]{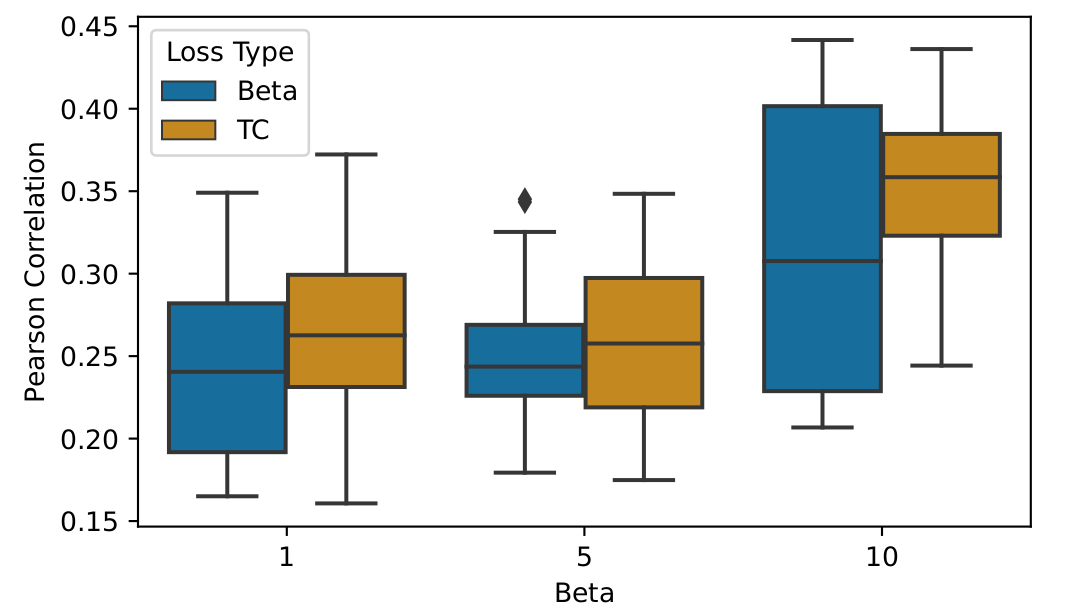}
		\end{minipage}
	}
	\vspace{-0.1in}
	\caption{Pearson correlation between saliency maps for different values of $\beta$.}
	\label{fig:vae_pearson}
	\vspace{-0.1in}	
\end{figure*}

In Definition~\ref{def:lf_feature_importance}, the inner product appearing in the label-free importance expression corresponds to a sum over the latent space dimensions. In practice, this has the effect of mixing the feature importance for each latent unit (neuron) to compute an overall feature importance. While this is reasonable when no particular meaning is attached to each latent unit, it might be inappropriate when the units are designed to be interpretable. Disantangled VAEs, for instance, involve latent units that are sensitive to change in single data generative factors, while invariants to other factors. This selective sensitivity permits to assign a meaning to each unit. An important question ensues: can we identify the generative factor associated to each latent unit by using their saliency maps? To answer, we propose a qualitative and quantitative analysis of the saliency maps from disentangled latent units. We show that, even for disentangled $\beta$-VAEs, the saliency maps of individual latent units are hard to interpret on their own.

\textbf{Setup.} We study two popular disentangled VAEs : the $\beta$-VAE~\cite{Higgins2017} and the TC-VAE~\cite{Chen2018}. Those two VAEs involve a variational encoder computing the expected representation $\vaemu: \X \rightarrow \H $ as well as its standard deviation $\vaesigma: \X \rightarrow \H$ and a decoder $\f_d : \H \rightarrow \X$. Latent samples are obtained with the reparametrization trick~\cite{Kingma2013}: $\h = \vaemu(\x) + \vaesigma(\x) \odot \noise$, $\noise \sim \mathcal{N}(\boldsymbol{0}, \boldsymbol{I})$. These VAEs are trained on the MNIST and dSprites datasets~\cite{dsprites17} ($90\% - 10 \%$ train-test split) to minimize their objective. We use $d_H = 3$ latent units for MNIST and $d_H = 6$ for dSprites. We train 20 disentangled VAEs of each type for $\beta \in \{1,5,10\}$.

\textbf{Qualitative Analysis.} We use Gradient Shap to evaluate the importance\footnote{In this case, we \emph{don't} use label-free feature importance $b_i$.} $a_i(\mu_j , \x)$ of each pixel $x_i$ from the image $\x$ to predict each latent unit $j \in [d_H]$. Again, we use the Pearson correlation to compare the saliency maps for each pair of latent unit~\footnote{We average the correlation over pairs of latent units.}. In this case, a low Pearson correlation corresponds to latent units paying attention to distinct parts of the images. Clearly, this is desirable if we want to easily identify the specialization of each latent unit. Therefore, use this criterion to select a VAE to analyse among the 120 VAEs we trained on each dataset. This corresponds to a $\beta$-VAE with $\beta = 10$ for MNIST and a TC-VAE with $\beta = 1$ for dSprites. We show the various saliency maps for 4 test images on Figure~\ref{fig:examples_vae}. The saliency maps appear to follow patterns that make the interpretation difficult. Here are a couple of examples that we can observe: \textbf{(1)} A latent unit is sensitive to a given image while insensitive to a similar image (e.g. Latent Unit 1 of the MNIST VAE is sensitive to Image 1 but not to Image 2). \textbf{(2)} The focus of a latent unit changes completely between two similar images (e.g. Latent Unit 4 of the dSprites VAE focuses on the interior of the square from Image 1 but only on the edges of the square from Image 2). \textbf{(3)} Several latent units focus on the same part of the image (e.g. Image 4 of MNIST and Image 3 of dSprites).  Additional examples can be found in Appendix~\ref{appendix:experiments_details}.   

\textbf{Quantitative Analysis.} If we increase the disentanglement of a VAE, does it imply that distinct latent units are going to focus on more distinct features? For the disantangled VAEs we consider, the strength of disentanglement increases with $\beta$. Hence, the above question can be formulated more precisely : does the correlation between the latent units saliency map decrease with $\beta$? To answer this question, we make box-plots of the Pearson correlation coefficients for both VAE types (Beta and TC) and various values of $\beta$. Results can be observed in Figure~\ref{fig:vae_pearson}. For MNIST, the Pearson correlation slightly decreases with $\beta$ (Spearman $\rho = -.15$). For dSprites, the Pearson correlation moderately increases with $\beta$ (Spearman $\rho =.43$). This analysis shows that increasing $\beta$ \emph{does not} imply that latent units are going to pay attention to distinct part of the images. In fact, the opposite is true for dSprites. If this results is surprising at first, it can be understood by thinking about disentanglement. As aforementioned, increasing disentanglement encourages distinct latent units to pay attention to distinct generative factors. There is no guarantee that distinct generative factors are associated to distinct features. To illustrate, let us consider two generative factors of the dSprites dataset: the position of a shape and its scale. Clearly, these two generative factors can be identified by paying attention to the edges of the shape appearing on the image, as various latent units do in Figure~\ref{subfig:examples_vae_dsprites}. Whenever generative factors are not unambiguously associated to a subset of features, latent units can identify distinct generative factors by looking at similar features. In this case, increasing disentanglement of latent units does not necessarily make their saliency maps more decorrelated. We conclude two things: \textbf{(1)} If we want to identify the role of each latent unit with their saliency maps, disentanglement might not be the right approach. Perhaps it is possible to introduce priors on the saliency map to control the features the model pays attention to, like it was done by \citet{Erion2021} in a supervised setting. We leave this for future works. \textbf{(2)} Taking a weighted sum of these saliency maps (as done by our label-free wrappers) does not sacrifice any interpretable information specific to each unit.

\section{Discussion} \label{sec:discussion}
We introduced label-free explainability, a new framework to extend linear feature importance and  example importance methods to the unsupervised setting. We showed that our framework guarantees crucial properties, such as completeness and invariance with respect to latent symmetries. We validated the framework on several datasets and verified that it permits to compare different representation spaces, both qualitatively and quantitatively. Finally, we challenged some common beliefs about the interpretability of $\beta$-VAEs. 

Label-free explainability opens up many interesting avenues for future work. A first one is the extension of loss-based example importance methods to contrastive losses, hence completing Section~\ref{sec:example}. Another one is to compare the representation learned by different state of the art encoders with the approach from Section~\ref{subsec:pretext}. A third one, as suggested in Section~\ref{subsec:vaes}, is to regularize latent units to make their individual saliency maps more interpretable. Finally, a more radical way to interpret representation spaces is to use symbolic regression~\cite{Crabbe2020Learning} to express latent units as closed form expressions of the input features.

\section*{Acknowledgements}
The authors are grateful to Zhaozhi Qian, Alicia Curth and the 3 anonymous ICML reviewers for their useful comments on an earlier version of the manuscript. Jonathan Crabbé is funded by Aviva and Mihaela van der Schaar by the Office of Naval Research (ONR), NSF 172251.

\nocite{*}
\clearpage
\bibliographystyle{icml2022}
\bibliography{main}

\newpage
\appendix
\onecolumn
\section{Properties of the Label-Free Extensions} \label{appendix:properties}
In this appendix, we prove the completeness property. Next, we motivate and prove the orthogonal invariance of our label-free extensions.

\subsection{Completeness}
Let us prove Proposition~\ref{proposition:completeness}.
\begin{proof}
	The proof is an immediate consequence of Definition~\ref{def:lf_feature_importance} and the completeness property of the feature importance score $a_i$:
	\begin{align*}
		\sum_{i=1}^{d_X} b_i(\f , \x) &\stackrel{\eqref{eq:lf_feature_importance}}{=} \sum_{i=1}^{d_X} a_i(g_{\x}, \x) \\
		&= g_{\x}(\x) - a_0,
	\end{align*}
where we used the completeness property to obtain the second equality and $a_0 \in \R$ is the baseline for the importance score $a_i$. By noting that $g_{\x}(\x) = \iprod{\f(\x) , \f(\x)} = \norm{\f(\x)}^2$, we obtain the desired equality \eqref{eq:lf_completeness} with the identification $b_0 = a_0$. 
\end{proof}

\subsection{Invariance with respect to latent symmetries}
\label{subappendix:invariance}
In Section~\ref{sec:introduction}, we described the ambiguity associated to the axes of the latent space $\H$. This line of reasoning can be made more formal with symmetries. Due to the fact that each axis of the latent space is not associated with a fixed and predetermined label, there exists many latent spaces that are equivalent to each other. For instance, if we swap two axes of the latent space by relabelling $h_1 \mapsto h_2$ and $h_2 \mapsto h_1$, we do not change the latent space structure. More generally, given an inner product $\iprod{\cdot , \cdot}$, the set of transformations that leave the geometry of the latent space $\H$ invariant is the set of orthogonal transformations.

\begin{definition}[Orthogonal Transformations]
	Let $\H$ be a real vector space equipped with an inner product $\iprod{\cdot , \cdot} : \H^2 \rightarrow \R$. An orthogonal transformation  is a linear map $\T : \H \rightarrow \H$ such that for all $\h_1 , \h_2 \in \H$, we have:
	\begin{align*}
		\iprod{\T(\h_1) , \T(\h_2)} = \iprod{\h_1 , \h_2}.
	\end{align*}
\end{definition} 
\begin{remark}
	In the case where $\iprod{\cdot , \cdot}$ is the standard euclidean inner product $\iprod{\h_1 , \h_2} = \h_1^\intercal \h_2$, the orthogonal transformations are represented by orthogonal matrices in $O(d_H) = \{ \boldsymbol{M} \in \R^{d_H \times d_H} \mid \boldsymbol{M}^\intercal \boldsymbol{M}  \}$. These transformations include rotations, axes permutations and mirror symmetries. 
\end{remark}

Of course, since these transformations leave the geometry of the latent space invariant, we would expect the same for the explanations. We verify that this is indeed the case for our label-free extension of feature importance.

\begin{proposition}[Label-Free Feature Importance Invariance]
	The label-free importance scores $b_i(\cdot , \cdot) , i \in [d_X]$ are invariant with respect to orthogonal transformations in the latent space $\H , \iprod{\cdot, \cdot}$. More formally, for all $\f \in \A(\H^{\X})$ , $\x \in \X$ and $i \in [d_X]$:
	\begin{align*}
		b_i(\T \circ \f , \x) = b_i(\f , \x),
	\end{align*}	
	where $\T$ is any orthogonal transformation of the latent space $\H , \iprod{\cdot, \cdot}$.
\end{proposition}
\begin{remark}
	This property is a further motivation for the usage of an inner product in Definition~\ref{def:lf_feature_importance}.
\end{remark}
\begin{proof}
	This proposition is a trivial consequence of the inner product appearing in Definition~\ref{def:lf_feature_importance}. Let $g^{\T}_{\x}$ be the auxiliary function associated to $\T \circ \f$ for some $\x \in \X$. We note that for all $\tilde{\x} \in \X$:
	\begin{align*}
		g_{\x}^{\T}(\tilde{\x}) &\stackrel{\eqref{eq:auxiliary_function}}{=} \iprod{\T \circ \f(\x), \T \circ \f(\tilde{\x})} \\
		&= \iprod{\f(\x),  \f(\tilde{\x})} \\
		&\stackrel{\eqref{eq:auxiliary_function}}{=} g_{\x}(\tilde{\x}),
	\end{align*}
	where we used the fact that $\T$ is orthogonal in the second equality and $g_{\x}(\tilde{\x})$ is the auxiliary function associated to $\f$. Since this holds for any $\tilde{\x} \in \X$, we have that $g_{\x}^{\T} = g_{\x}$ for all $\x \in \X$. This allows us to write:
	\begin{align*}
		b_i(\T \circ \f, \x) &\stackrel{\eqref{eq:lf_feature_importance}}{=} a_i(g^{\T}_{\x}, \x) \\
		&= a_i(g_{\x}, \x) \\
		&\stackrel{\eqref{eq:lf_feature_importance}}{=} 	b_i(\f, \x)
	\end{align*}
	for all $\x \in \X$ and $i \in [d_X]$. This is the desired identity.
\end{proof}

When it comes to example importance methods, the same guarantee holds for representation-based methods:

\begin{proposition}[Representation-Based Example Importance Invariance]
	Let  $\H , \iprod{\cdot, \cdot}$ be a latent space. 
	The label-free importance scores $c^n(\cdot , \cdot) , n \in \left[N\right]$ outputted by DKNN~\eqref{eq:dknn} are invariant with respect to orthogonal transformations of $\H$ if they are defined with a kernel $\kappa : \H^2 \rightarrow \R^+$ that is invariant with respect to orthogonal transformations:
	\begin{align*}
		\kappa(\T(\h_1) , \T(\h_2)) = \kappa(\h_1 , \h_2)
	\end{align*}
	for all orthogonal transformation $\T$ and $\h_1, \h_2 \in \H$.
	Similarly, the importance scores outputted by SimplEx~\eqref{eq:simplex} are invariant with respect to orthogonal transformations of $\H$. In both cases, the invariance property can be written more formally: for all $\f \in \A(\H^{\X})$ , $\x \in \X$ and $n \in \left[N\right]$:
	\begin{align*}
		c^n(\T \circ \f , \x) = c^n(\f , \x)
	\end{align*}
\end{proposition}
\begin{remark}
	The invariance property for the kernel function is verified for kernels that involve the inner-product $\iprod{\cdot , \cdot}$ in their definition. This includes RBF, Matern and Polynomial Kernels~\cite{Rasmussen2003}. Note that constant kernels trivially verify this property. Finally, replacing the kernel function by the inverse-distance in latent space (as it is done in our implementation) $\kappa(\h_1 , \h_2) = \norm{\h_1 - \h_2}^{-1}$ also preserves the invariance property.  
\end{remark}
\begin{proof}
	We start by noting that the latent space distance is invariant under orthogonal transformations. for all $\h_1 , \h_2 \in \H$:
	\begin{align*}
		\norm{\T(\h_1) - \T(\h_2)}^2 &= \norm{\T(\h_1 - \h_2)}^2\\
		&= \iprod{\T(\h_1 - \h_2), \T(\h_1 - \h_2)}\\
		&= \iprod{\h_1 - \h_2, \h_1 - \h_2} \\
		&= \norm{\h_1 - \h_2}^2,
	\end{align*}
	where we successively used the linearity and orthogonality of $\T$. Note that this equation is equivalent to $	\norm{\T(\h_1) - \T(\h_2)} = \norm{\h_1 - \h_2}$ as both norms are positive. Since the latent KNNs in \eqref{eq:dknn} are computed with this latent space distance, we deduce their invariance under orthogonal transformations. By combining this to the invariance of the kernel, we obtain the desired invariance for the DKNN importance scores~\eqref{eq:dknn}. For all $\x \in \X$, $\f \in \A(\H^{\X})$ and orthogonal transformation $\T$:
	\begin{align*}
		c_{\textrm{DKNN}}^n(\T \circ \f , \x) &\stackrel{\eqref{eq:dknn}}{=} \boldsymbol{1}\left[n \in \textrm{KNN}(\x)\right] \cdot \kappa \left[ \T \circ \f(\x^n) , \T \circ \f(\x) \right] \\
		&= \boldsymbol{1}\left[n \in \textrm{KNN}(\x)\right] \cdot \kappa \left[ \f(\x^n) , \f(\x) \right] \\
		&\stackrel{\eqref{eq:dknn}}{=} c_{\textrm{DKNN}}^n(\f , \x),
	\end{align*}
	where we have used the invariance property to obtain the second equality. We can proceed similarly for SimplEx~\eqref{eq:simplex}:
	\begin{align*} 
		c_{\textrm{SimplEx}}^n(\T \circ \f , \x) &\stackrel{\eqref{eq:simplex}}{=}\argmin_{\boldsymbol{\lambda} \in [0,1]^N} \norm{\T \circ \f(\x) - \sum_{n=1}^N \lambda^n \T \circ \f(\x^n)} \\
		&= \argmin_{\boldsymbol{\lambda} \in [0,1]^N} \norm{\T\left[ \f(\x) - \sum_{n=1}^N \lambda^n  \f(\x^n)\right]} \\
		&= \argmin_{\boldsymbol{\lambda} \in [0,1]^N} \norm{\f(\x) - \sum_{n=1}^N \lambda^n  \f(\x^n)} \\
		&\stackrel{\eqref{eq:simplex}}{=} c_{\textrm{SimplEx}}^n(\f , \x),
	\end{align*}
	where we have successively used the linearity and orthogonality of $\T$. Those are the desired identities.
\end{proof}

The only label-free extension that we have not yet discussed are the loss-based example importance methods from Section~\ref{subsec:loss_based}. Unfortunately, due to the fact that the black-box $\f$ is only a component required in the evaluation of the loss $L$, it is not possible to provide a general guarantee like in the previous examples. If we take the example of the autoencoder $\f_d \circ \f_e$, we note that applying an orthogonal transformation $\f_e \mapsto \T \circ \f_e$ to the encoder leaves the autoencoder invariant only if this transformation is undone by the decoder $\f_d \mapsto \f_d \circ \T^{-1}$. Unlike the other methods, the invariance of loss-based example importance scores therefore requires restrictive assumptions. If invariance of the explanations under orthogonal transformations is required, this might be an argument in favour of representation-based methods.
\section{Implementation Details} \label{appendix:implementation}
In this appendix, we detail the implementation of our label-free extensions. 

\subsection{Label-Free Feature Importance}
The label-free feature importance methods used in our experiments are described in Table~\ref{tab:feature_importance_methods}:

\begin{table}[h]
	\setstretch{1.5}
	\caption{Feature Importance Methods.}
	\label{tab:feature_importance_methods}
	\begin{center}
		\begin{adjustbox}{width=\columnwidth}
			\begin{tabular}{c c c c l} 
				\toprule
				Method & Ref. & Linearity & Completeness &  Label-Free Expression\\
				\hline
				Saliency & \cite{Simonyan2013} & \cmark & \xmark & $b_i(\f , \x) = \partder{g_{\x}}{x_i}(\x)$\\
				Integrated Gradients & \cite{Sundararajan2017} & \cmark & \cmark & $b_i(\f , \x) = (x_i - \bar{x}_i) \int_0^1 \partder{g_{\x}}{x_i}(\x + \tau(\x - \bar{\x})) d\tau$ \\
				Gradient Shap & \cite{Lundberg2017} & \cmark & \cmark & $b_i(\f , \x) = (x_i - \bar{x}_i) \E_{\noise, U} \left[ \partder{g_{\x}}{x_i}(\x + \noise + U (\x - \bar{\x})) \right]$ \\
				DeepLift & \cite{Shrikumar2017} & \cmark & \cmark & $b_i(\f , \x) = C_{\Delta x_i \Delta g_{\x} }$ \\
				
				\bottomrule
			\end{tabular}
		\end{adjustbox}
	\end{center}
\end{table}

where $\bar{\x} \in \X$ is a baseline input, $\noise \sim \mathcal{N}(\boldsymbol{0}, \boldsymbol{I})$, $U \sim \textrm{Uniform}(0,1)$ and $C_{\Delta x_i \Delta g_{\x} }$ is used by propagating the Deeplift rules along the computational graph of $g_{\x}$. Note that, in each case, partial derivatives are computed with respect to the argument of $g_{\x}$ only (hence we do not consider derivatives of the form $g_{\partder{\x}{x_i}}$). We use the Captum~\cite{Kokhlikyan2020} implementation of each method.

To extend this implementation to the label-free setting, it is necessary to define an auxiliary function $g_{\x}$ associated to the vectorial black-box function $\f$ for each testing example $\x \in \Dtest$. With libraries such as Pytorch, it is possible to define an auxiliary function as a wrapper around the module that represents $\f$. This allows us to compute the importance scores with a single batch call of the original feature importance method, as described in Algorithm~\ref{alg:lf_feature_importance}.

\begin{algorithm}[h]
	\setstretch{1.5}
	\caption{Label-Free Feature Importance}
	\label{alg:lf_feature_importance}
	\begin{algorithmic}
		\STATE {\bfseries Input:} Batch $\rx \in \X^B$ of size $B$, Black-box $\f: \X \rightarrow \H$, Feature importance method $a_i(\cdot, \cdot): \A(\H^{\X}) \times \X^B \rightarrow \R^B$
		\STATE {\bfseries Output:} Batch label-free feature importance $b_i(\f , \rx)$.
		\STATE Define batch auxiliary function $g_{\rx}$ as a wrapper around $\f$ according to \eqref{eq:auxiliary_function}.
		\STATE Compute the label-free importance scores for the batch $b_i(\f , \rX) \leftarrow a_i(g_{\rx}, \rx)$.
	\end{algorithmic}
\end{algorithm}

\subsection{Label-Free Example Importance}

We detail the label-free example importance methods used in our experiments in Table~\ref{tab:example_importance_methods}:

\begin{table}[H]
\setstretch{2}	
\caption{Example Importance Methods.}
\label{tab:example_importance_methods}
\begin{center}
	\begin{adjustbox}{width=\columnwidth}
		\begin{tabular}{c | c c l} 
			\toprule
			Method Type & Method & Ref. & Label-Free Expression\\
			\hline
			\multirowcell{2}{Loss-Based} & Influence Functions & \cite{Koh2017} & $c^n(\f, \x) = \frac{1}{N} \iprod[\relevant{\Theta}]{\grad[\relevant{\param}] L(\x, \params) ,  \hessian[\relevant{\param}]^{-1} \grad[\relevant{\param}] L(\x^n, \params) }$ \\
			& TracIn & \cite{Pruthi2020} & $c^n(\f, \x) = \sum_{t=1}^T \eta_t \iprod[\relevant{\Theta}]{\grad[\relevant{\param}] L(\x, \param_t) ,  \grad[\relevant{\param}] L(\x^n, \param_t) }$ \\
			\hline
			\multirow{3}*{Representation-Based} & Deep K-Nearest Neighbours & \cite{Papernot2018} & $c^n(\f, \x)= \boldsymbol{1}\left[n \in \textrm{KNN}(\x)\right] \cdot \kappa \left[ \f(\x^n) , \f(\x) \right]$ \\
			& SimplEx & \cite{Crabbe2021} & \makecell[l]{$c^n(\f, \x) = \argmin_{\boldsymbol{\lambda} \in [0,1]^N} \norm{\f(\x) - \sum_{n=1}^N \lambda^n  \f(\x^n)}$ \\ $\textrm{s.t. } \sum_{n=1}^N \lambda^n = 1$} \\
			\bottomrule
		\end{tabular}
	\end{adjustbox}
\end{center}
\end{table}

where $\relevant{\param}$ are the parameters of the black-box $\f$. Our implementation closely follows the above references with some subtle differences. For completeness, we detail the algorithm used for each method. We start with Influence Functions in Algorithm~\ref{alg:lf_influence_functions}.

\begin{algorithm}[H]
	\setstretch{1.5}
	\caption{Label-Free Influence Functions}
	\label{alg:lf_influence_functions}
	\begin{algorithmic}
		\STATE {\bfseries Input:} Test input $\x \in \X$, Black-box $\f: \X \rightarrow \H$, Optimized parameters $\params \in \Theta$, Relevant black-box parameters $\relevant{\param} \in \relevant{\Theta}$, Loss function $L: \X \times \Theta \rightarrow \R$ used to train the black-box, Training set $\Dtrain = \{ \x^n \mid n \in [N]\}$, Number of samples $S \in \N^*$, Number of recursions $R \in \N^*$.
		\STATE {\bfseries Output:} Label-free influence functions $c^n(\f , \x)$.
		\STATE Initialize $\boldsymbol{v}_0 \leftarrow \grad[\relevant{\param}] L(\x, \params)$.
		\STATE Initialize $\boldsymbol{v} \leftarrow \boldsymbol{v}_0$.
		\FOR{recursion in $[R]$}
		\STATE Sample $S$ training points $\x^{n_1}, \dots, \x^{n_S}$ from the training set $\Dtrain$.
		\STATE Make a Monte-Carlo estimation of the training loss Hessian $\boldsymbol{H} \leftarrow \frac{1}{S} \grad[\relevant{\param}]^2 \sum_{s=1}^{S} L(\x^{n_s}, \params)$.
		\STATE Update the estimate for the inverse Hessian-vector product $\boldsymbol{v} \leftarrow \boldsymbol{v}_0 + (\boldsymbol{I} - \boldsymbol{H}) \boldsymbol{v}$.
		\ENDFOR
		\STATE Compute the influence function $c^n(\f , \x) \leftarrow \boldsymbol{v}^{\intercal} \grad[\relevant{\param}] L(\x^n, \params)$
	\end{algorithmic}
\end{algorithm}

This implementation follows the original implementation by \citet{Koh2017} that leverages the literature on second-order approximation techniques~\cite{Pearlmutter1994, Agarwal2016}. Note that Monte-Carlo estimations of the Hessian quickly become expensive when the number of model parameters grows. Due to our limited computing infrastructure, we limit the number of recursions to $R=100$. Furthermore, we only compute influence functions for smaller subsets of the training and testing set. The label-free version of TracIn is described in Algorithm~\ref{alg:lf_tracin}. In our implementation, we create a checkpoint after each interval of 10 epochs during training.

\begin{algorithm}[h]
	\setstretch{1.5}
	\caption{Label-Free TracIn}
	\label{alg:lf_tracin}
	\begin{algorithmic}
		\STATE {\bfseries Input:} Test input $\x \in \X$, Black-box $\f: \X \rightarrow \H$, Checkpoint parameters $\{\param_t \in \Theta \mid t \in [T] \}$, Relevant black-box parameters $\relevant{\param} \in \relevant{\Theta}$, Loss function $L: \X \times \Theta \rightarrow \R$ used to train the black-box, Training set $\Dtrain = \{ \x^n \mid n \in [N]\}$, Checkpoint learning rates $\{\eta_t \in \R \mid t \in [T] \}$.
		\STATE {\bfseries Output:} Label-free TraceIn scores $c^n(\f , \x)$.
		\STATE Initialize $c \leftarrow 0$.
		\FOR{$t$ in $[T]$}
		\STATE Update the estimate  $c \leftarrow c + \eta_t \grad[\relevant{\param}]^{\intercal} L(\x^n, \param_t) \grad[\relevant{\param}] L(\x, \param_t)$.
		\ENDFOR
		\STATE Return the score $c^n(\f , \x) \leftarrow c$
	\end{algorithmic}
\end{algorithm}

 When it comes to DKNN, the formula~\eqref{eq:dknn} can be computed explicitly without following a particular procedure. In our implementation, we replaced the kernel function by an inverse distance $\kappa(\h_1 , \h_2) = \norm{\h_1 - \h_2}^{-1}$. Further, to make it more fair with other baselines that assign a score to each examples (and not only to $K \in \N^*$ examples), we removed the indicator in \eqref{eq:dknn}: $\boldsymbol{1}[n \in \textrm{KNN}(x)] \mapsto 1$. In this way, the $K$ most important examples always correspond to the $K$ nearest neighbours. Finally, the label-free version of SimplEx is described in Algorithm~\ref{alg:lf_simplex}.  

\begin{algorithm}[h]
	\setstretch{1.5}
	\caption{SimplEx}
	\label{alg:lf_simplex}
	\begin{algorithmic}
		\STATE {\bfseries Input:} Test input $\x \in \X$, Black-box $\f: \X \rightarrow \H$, Training set $\Dtrain = \{ \x^n \mid n \in [N]\}$, Number of epochs $E \in \N^*$.
		\STATE {\bfseries Output:} SimplEx scores $c^n(\f , \x)$.
		\STATE Initialize weights $(w^n)_{n=1}^N \leftarrow \boldsymbol{0}$.
		\FOR{epoch in $[E]$}
		\STATE Normalize weights $(\lambda^n)_{n=1}^N \leftarrow \textrm{Softmax}\left[ (w^n)_{n=1}^N \right]$
		\STATE Estimate Error $\mathcal{L} = \norm{\f(\x) - \sum_{n=1}^N \lambda^n \f(\x^n)}$.
		\STATE Update weights with Adam  $(w^n)_{n=1}^N \leftarrow \textrm{Adam Step}(\mathcal{L})$.
		\ENDFOR
		\STATE Return the score $c^n(\f , \x) \leftarrow \lambda^n$
	\end{algorithmic}
\end{algorithm}

Note that our implementation of SimplEx is identical to the original one. It relies on an Adam~\cite{Kingma2014} with the default Pytorch parameters ($ \textrm{learning rate} = .001, \beta_1 = .9, \beta_2 = .999, \epsilon = 10^{-8}, \textrm{weight decay} = 0$).

In this section, we have presented many label-free implementations of feature and example importance methods. For some types of explanations, like counterfactual explanations~\cite{Wachter2017}, the label plays an essential role. Hence, it does not always make sense to extend an explanation to the label-free setting. 
\section{Experiments Details} \label{appendix:experiments_details}
In this appendix, we provide further details to support the experiments described in Section~\ref{sec:experiments}. All our experiments have been performed on a machine with Intel(R) Core(TM) i5-8600K CPU @ 3.60GHz [6 cores] and Nvidia GeForce RTX 2080 Ti GPU. Our implementation is done with Python 3.8 and  Pytorch 1.10.0.  

\subsection{Consistency Checks} 
We provide some details for the experiments in Section~\ref{subsec:consistency_checks}. 

\paragraph{ECG5000 dataset.} The dataset $\D$ contains 5000 univariate time series\footnote{Note that $t$ is used to index the time series steps, as opposed to model checkpoints in Section~\ref{subsec:loss_based}.} $(x_t)_{t=1}^T \in \X^T$ describing the heartbeat of a patient. Each time series describes a single heartbeat with a resolution of $T=140$ time steps. For the sake of notation, we will represent univariate time series by vectors: $\x = (x_t)_{t=1}^T$. Each time series comes with a label $y \in \{0,1\}$ indicating if the heartbeat is normal ($y=0$) or abnormal ($y=1$). Since it is laborious to manually annotate 5000 time series, those labels were generated automatically. Of course, those labels are not going to be used in training our model. We only use the labels to perform consistency checks once the model has been trained. 

\paragraph{MNIST autoencoder.} We use a denoising autoencoder $\f_d \circ \f_e : \X \rightarrow \X$ that consists in an encoder $\f_e : \X \rightarrow \H$ and a decoder $\f_d : \H \rightarrow \X$ with $d_X = 28^2$, $d_H = 4$. The architecture for the autoencoder is described in Table~\ref{tab:autoencoder_mnist}. We corrupt each training image $\x \in \Dtrain$ with random noise $\noise \sim \mathcal{N}(\boldsymbol{0}, \nicefrac{\boldsymbol{I}}{3})$ where $\boldsymbol{I}$ is the identity matrix on $\X$. The autoencoder is trained to minimize the denoising loss $L_{\textrm{den}}(\x) = \E_{\noise} [\x - \f_d \circ \f_e (\x + \noise) ]^2$. The autoencoder is trained for 100 epochs with patience 10 by using Pytorch's Adam with hyperparameters: $ \textrm{learning rate} = .001, \beta_1 = .9, \beta_2 = .999, \epsilon = 10^{-8}, \textrm{weight decay} = 10^{-5}$. The testing set is sometimes used for early stopping. This is acceptable because assessing the generalization of the learned model is \emph{not} the focus of our paper. Rather, we only use the test set to study the explanations of the learned model.

\paragraph{ECG5000 autoencoder.} \parapoint{Feature importance:}  We train a reconstruction autoencoder $\f_d \circ \f_e : \X^T \rightarrow \X^T$ that consists in an encoder $\f_e : \X^T \rightarrow \H$ and a decoder $\f_d : \H \rightarrow \X^T$ with $d_H = 64$. This model is trained with a training set $\Dtrain$ of 2919 time series from $\D$ that correspond to \emph{normal} heartbeats: $y = 0$. In this way, the testing set $\Dtest = \D \setminus \Dtrain$ contains only abnormal heartbeats: $y=1$. The model is trained to minimize the reconstruction loss $L_{\textrm{rec}}(\x) = \sum_{t=1}^{T} \left| x_t - \left[\f_{d} \circ \f_e (\x)\right]_t \right|$. The autoencoder is trained for 150 epochs with patience 10 by using Pytorch's Adam with hyperparameters: $ \textrm{learning rate} = .001, \beta_1 = .9, \beta_2 = .999, \epsilon = 10^{-8}, \textrm{weight decay} = 0$. Its detailed architecture is presented in Table~\ref{tab:autoencoder_ecg5000}. \parapoint{Example importance:} We use the autoencoder described in Table~\ref{tab:autoencoder_ecg5000} with $d_H = 32$. The whole training process is identical to the one for feature importance with one difference: $\Dtrain$ and $\Dtest$ are now obtained with a random split of $\D$ ($80\% - 20\%$). This means that both subsets contain normal and abnormal heartbeats.

\paragraph{CIFAR-10 SimCLR.} We use a SimCLR network $\f_p \circ \f_e : \X \rightarrow \mathcal{P}$ that consists in a Resnet-18 encoder $\f_e : \X \rightarrow \H$ and a multilayer perceptron projection head $\f_p : \H \rightarrow \mathcal{P}$ with $d_X = 3 \cdot 32^2$, $d_H = 512$. The architecture for the SimCLR network is described in Table~\ref{tab:simclr_cifar10}. We use SimCLR's contrastive loss to train the model~\cite{Chen2020}. The model is trained for 100 epochs by using Pytorch's SGD with hyperparameters: $ \textrm{learning rate} = .6, \textrm{momentum} = .9, \textrm{weight decay} = 10^{-6}$.

\begin{table}[h]
	\setstretch{1.5}
	\caption{MNIST Autoencoder Architecture.}
	\label{tab:autoencoder_mnist}
	\begin{center}
		\begin{adjustbox}{width=\columnwidth}
			\begin{tabular}{c|clc} 
				\toprule
				Component & Layer Type & Hyperparameters & Activation Function\\
				\hline
				\multirow{7}*{Encoder} & Conv2d & \makecell[l]{Input Channels:1 ; Output Channels:8 ; Kernel Size:3 ; Stride:2 ; Padding:1 } & ReLU  \\ 
				& Conv2d & \makecell[l]{Input Channels:8 ; Output Channels:16 ; Kernel Size:3 ; Stride:2 ; Padding:1 } & ReLU \\ 
				& BatchNorm & Input Channels:16 & ReLU  \\
				& Conv2d & \makecell[l]{Input Channels:16 ; Output Channels:32 ; Kernel Size:3 ; Stride:2 ; Padding:0} & ReLU  \\  
				& Flatten & Start Dimension:1 &   \\ 
				& Linear & Input Dimension: 288 ; Output Dimension: 128 & ReLU  \\
				& Linear & Input Dimension: 128 ; Output Dimension: 4 &  \\ 
				\hline
				\multirow{8}*{Decoder} & Linear & Input Dimension: 4 ; Output Dimension: 128 & ReLU  \\
				& Linear & Input Dimension: 128 ; Output Dimension: 288 & ReLU  \\
				& Unflatten & Dimension:1 ; Unflatten Size:(32, 3, 3) &   \\
				& ConvTranspose2d & \makecell[l]{Input Channels:32 ; Output Channels:16 ; Kernel Size:3 ; Stride:2 ; Output Padding:0} &   \\  
				& BatchNorm & Input Channels:16 & ReLU  \\
				& ConvTranspose2d & \makecell[l]{Input Channels:16 ; Output Channels:8 ; Kernel Size:3 ; Stride:2 ; Output Padding:1} &   \\
				& BatchNorm & Input Channels:8 & ReLU  \\
				& ConvTranspose2d & \makecell[l]{Input Channels:8 ; Output Channels:1 ; Kernel Size:3 ; Stride:2 ; Output Padding:1} & Sigmoid  \\
				\bottomrule
			\end{tabular}
		\end{adjustbox}
	\end{center}
\end{table}

\begin{table}[h]
	\setstretch{1.5}
	\caption{ECG5000 Autoencoder Architecture.}
	\label{tab:autoencoder_ecg5000}
	\begin{center}
		\begin{adjustbox}{width=.8\columnwidth}
			\begin{tabular}{c|clc} 
				\toprule
				Component & Layer Type & Hyperparameters & Activation Function\\
				\hline
				\multirow{2}*{Encoder} & LSTM & \makecell[l]{Input Size:1 ; Hidden Size: $2 \cdot d_H$} &   \\ 
				& LSTM & \makecell[l]{Input Size: $2 \cdot d_H$ ; Hidden Size:$d_H$} &   \\ 
				\hline
				Representation & \multicolumn{3}{l}{\makecell[l]{The latent representation $\h$ is given by the output of the second LSTM at the last time step. \\ This representation is copied at each time step to be a valid input for the first decoder LSTM.}} \\
				\hline
				\multirow{3}*{Decoder} & LSTM & \makecell[l]{Input Size: $d_H$ ; Hidden Size: $d_H$} &   \\ 
				& LSTM & \makecell[l]{Input Size: $d_H$ ; Hidden Size: $2 \cdot d_H$} &   \\ 
				& Linear & Input Dimension: $2 \cdot d_H$ ; Output Dimension: 1 & \\
				\bottomrule
			\end{tabular}
		\end{adjustbox}
	\end{center}
\end{table}

\begin{table}[h]
	\setstretch{1.5}
	\caption{CIFAR-10 ResNet-18  Architecture.}
	\label{tab:simclr_cifar10}
	\begin{center}
		\begin{adjustbox}{width=.8\columnwidth}
			\begin{tabular}{c|clc} 
				\toprule
				Component & Layer Type & Hyperparameters & Activation Function\\
				\hline
				Encoder & ResNet-18 & \makecell[l]{Similar to Appendix B.9 in \cite{Chen2020}.} &   \\ 
				\hline
				\multirow{2}*{Projection Head} & Linear & \makecell[l]{Input Dimension: 512 ; Output Dimension: 2048} &  ReLU \\ 
				& Linear & Input Dimension: 2048 ; Output Dimension: 128 & \\
				\bottomrule
			\end{tabular}
		\end{adjustbox}
	\end{center}
\end{table}

\paragraph{Feature Importance.} As a baseline for MNIST feature importance, we use a black image $\bar{\x} = \boldsymbol{0}$. For ECG5000, we use the average normal heartbeat as a baseline: $\bar{\x} = \nicefrac{\sum_{\x \in \Dtrain} \x}{\left| \Dtrain \right|}$. For CIFAR-10, we use a blurred version of the image $\x$ we wish to explain as a baseline: $\bar{\x} = \boldsymbol{G}_{\sigma} \otimes \x$, where $\boldsymbol{G}_{\sigma}$ is a Gaussian blur with kernel of size 21 with width $\sigma = 5$ and $\otimes$ denotes the convolution operator. 

\paragraph{ROAR Test.} We perform the ROAR test~\cite{Hooker2018} for
our label-free feature importance methods. The setup is similar to Section~\ref{subsec:consistency_checks} except that the most important pixels are removed and a new autoencoder is fitted on the
ablated dataset. We report the results in Figure~\ref{fig:roar}. Again, the label-free feature importance methods discover pixels that increase the test loss more substantially than random pixels when removed, which supports the results from the main paper.

\begin{figure*}[h] 
	\vskip 0in
	\begin{center}
		\centerline{\includegraphics[width=.5\linewidth]{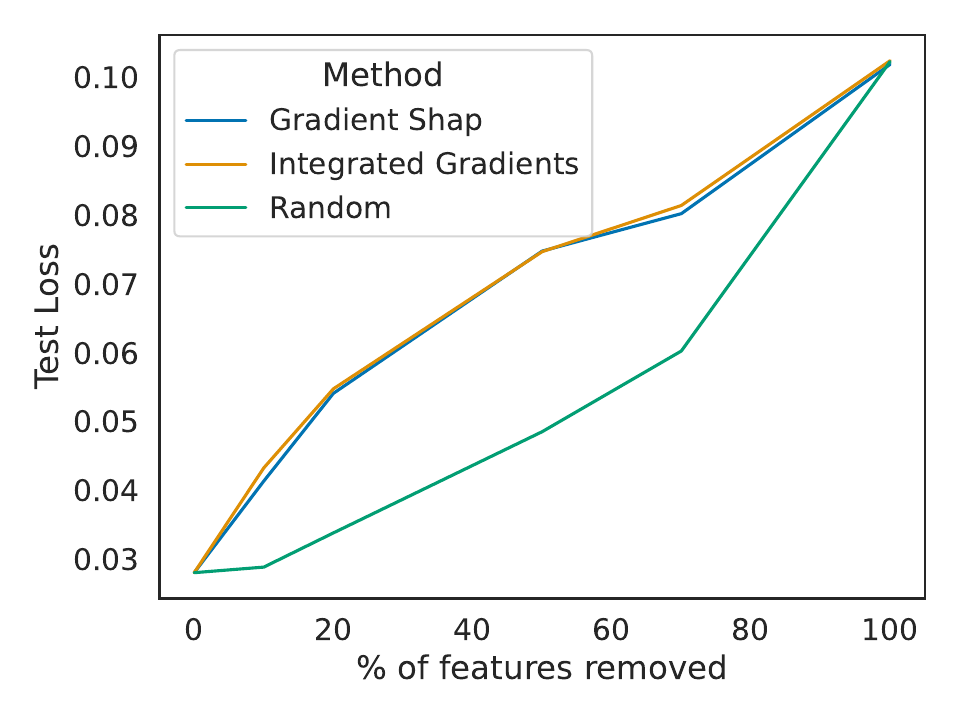}}
		\vspace{-0.1in}
		\caption{ROAR test for MNIST denoising autoencoder.}
		\label{fig:roar}
	\end{center}
	\vskip -0.3in
\end{figure*}

\subsection{Pretext Task Sensitivity} 

We provide some details for the experiments in Section~\ref{subsec:pretext}.

\paragraph{Models.} All the autoencoders have the architecture described in Table~\ref{tab:autoencoder_mnist}. The classifier has all the layers from the encoder in Table~\ref{tab:autoencoder_mnist} with an extra linear layer (Input Dimension:4 ; Output Dimension:10 ; Activation:Softmax) that converts the latent representations to class probabilities. Hence, it can be written as $\f_{l} \circ \f_{e}$, where $\f_{l} : \X \rightarrow \Y$ is an extra linear layer followed by a softmax activation. The reconstruction autoencoder is trained to minimize the reconstruction loss $L_{\textrm{rec}}(\x) = [\x - \f_d \circ \f_e (\x ) ]^2$. The inpainting autoencoder is trained to minimize the inptaiting loss $L_{\textrm{in}}(\x) = \E_{\rm} [\x - \f_d \circ \f_e (\rm \odot \x ) ]^2$, where $\rm$ is a random mask with $M_i \sim \textrm{Bernoulli}(0.8)$ for all $i \in [d_X]$. The classifier is trained to minimize the cross-entropy loss $L_{CE}(\x, \y) = - \y \odot \log \left[\f_l \circ \f_e(\x)\right]$, where $\y$ is the one-hot encoded label associated to the training example $\x$.  All the models are trained to minimize their objective for 100 epochs with patience 10 by using Pytorch's Adam with hyperparameters: $ \textrm{learning rate} = .001, \beta_1 = .9, \beta_2 = .999, \epsilon = 10^{-8}, \textrm{weight decay} = 10^{-5}$.  

\paragraph{Feature Importance.} As a baseline for the feature importance methods, we use a black image $\bar{\x} = \boldsymbol{0}$.

\paragraph{Metrics.} We use the Pearson coefficient to measure the correlation between two importance scores given a random test example and a random feature/training example. In our experiment, we compute the Pearson correlation between the label-free feature importance scores $b_i$ outputted by two different encoder $\f_{e1}, \f_{e2} : \X \rightarrow \H$:
\begin{align*}
	&r_{\textrm{feat.}}(\f_{e1}, \f_{e2}) = \frac{\cov_{\rx, I} \left[ b_I(\f_{e1}, \rx) , b_I(\f_{e2}, \rx) \right]}{\sigma_{\rx, I}\left[ b_I(\f_{e1}, \rx) \right] \sigma_{\rx, I}\left[ b_I(\f_{e2}, \rx) \right]} \\
	&\rx \sim \textrm{Empirical Distribution}(\Dtest) , I \sim \textrm{Uniform}([d_X]),
\end{align*}
where $\cov$ denotes the covariance between two random variables and  and $\sigma$ denotes the standard deviation of a random variable. Similarly, for label-free example importance scores $c^n$:
\begin{align*}
	&r_{\textrm{ex.}}(\f_{e1}, \f_{e2}) = \frac{\cov_{\rx, I} \left[ c^I(\f_{e1}, \rx) , c^I(\f_{e2}, \rx) \right]}{\sigma_{\rx, I}\left[ c^I(\f_{e1}, \rx) \right] \sigma_{\rx, I}\left[ c^I(\f_{e2}, \rx) \right]} \\
	&\rx \sim \textrm{Empirical Distribution}(\Dtest) , I \sim \textrm{Uniform}(\mathcal{J}),
\end{align*}
where $\mathcal{J} \subset [N]$ is the indices of the sampled training examples for which the example importance is computed. Those two Pearson correlation coefficients are the one that we report in Tables~\ref{tab:pretext_features_pearson}~and~\ref{tab:pretext_examples_pearson}.

\paragraph{Supplementary Examples.} To check that the qualitative analysis from Section~\ref{subsec:pretext} extends beyond the examples showed in the main paper, the reader can refer to Figures~\ref{fig:pretext_saliency_appendix}~and~\ref{fig:pretext_example_appendix}.

\subsection{Challenging our assumptions with disentangled VAEs} 
\label{subappendix:vae_complement}

We provide some details for the experiments in Section~\ref{subsec:vaes}.

\begin{table}[h]
	\caption{MNIST Variational Autoencoder Architecture.}
	\label{tab:vae_mnist}
	\begin{center}
		\setstretch{1.5}
		\begin{adjustbox}{width=\columnwidth}
			\begin{tabular}{c|clc} 
				\toprule
				Component & Layer Type & Hyperparameters & Activation Function\\
				\hline
				\multirow{7}*{Encoder} & Conv2d & \makecell[l]{Input Channels:1 ; Output Channels:32 ; Kernel Size:4 ; Stride:2 ; Padding:1 } & ReLU  \\ 
				& Conv2d & \makecell[l]{Input Channels:32 ; Output Channels:32 ; Kernel Size:4 ; Stride:2 ; Padding:1 } & ReLU \\ 
				& Conv2d & \makecell[l]{Input Channels:32 ; Output Channels:32 ; Kernel Size:4 ; Stride:2 ; Padding:1 } & ReLU \\  
				& Flatten & Start Dimension:1 &   \\ 
				& Linear & Input Dimension: 512 ; Output Dimension: 256 & ReLU  \\
				& Linear & Input Dimension: 256 ; Output Dimension: 256 & ReLU \\ 
				& Linear & Input Dimension: 256 ; Output Dimension: 6 & ReLU \\ 
				\hline
				Reparametrization Trick & \multicolumn{3}{l}{\makecell[c]{The output of the encoder contains $\vaemu$ and $\log \vaesigma$. \\ The latent representation is then generated via $\h = \vaemu(\x) + \vaesigma(\x) \odot \noise$, $\noise \sim \mathcal{N}(0, I)$}} \\
				\hline
				\multirow{8}*{Decoder} & Linear & Input Dimension: 3 ; Output Dimension: 256 & ReLU  \\
				& Linear & Input Dimension: 256 ; Output Dimension: 256 & ReLU  \\
				& Linear & Input Dimension: 256 ; Output Dimension: 512 & ReLU  \\
				& Unflatten & Dimension:1 ; Unflatten Size:(32, 4, 4) &   \\
				& ConvTranspose2d & \makecell[l]{Input Channels:32 ; Output Channels:32 ; Kernel Size:4 ; Stride:2 ; Output Padding:1} &  ReLU \\  
				& ConvTranspose2d & \makecell[l]{Input Channels:32 ; Output Channels:32 ; Kernel Size:4 ; Stride:2 ; Output Padding:1} & ReLU  \\
				& ConvTranspose2d & \makecell[l]{Input Channels:32 ; Output Channels:1 ; Kernel Size:4 ; Stride:2 ; Output Padding:1} & Sigmoid  \\
				\bottomrule
			\end{tabular}
		\end{adjustbox}
	\end{center}
\end{table}

\begin{table}[h]
	\setstretch{1.5}
	\caption{dSprites Variational Autoencoder Architecture.}
	\label{tab:vae_dsprites}
	\begin{center}
		\begin{adjustbox}{width=\columnwidth}
			\begin{tabular}{c|clc} 
				\toprule
				Component & Layer Type & Hyperparameters & Activation Function\\
				\hline
				\multirow{7}*{Encoder} & Conv2d & \makecell[l]{Input Channels:1 ; Output Channels:32 ; Kernel Size:4 ; Stride:2 ; Padding:1 } & ReLU  \\ 
				& Conv2d & \makecell[l]{Input Channels:32 ; Output Channels:32 ; Kernel Size:4 ; Stride:2 ; Padding:1 } & ReLU \\ 
				& Conv2d & \makecell[l]{Input Channels:32 ; Output Channels:32 ; Kernel Size:4 ; Stride:2 ; Padding:1 } & ReLU \\ 
				& Conv2d & \makecell[l]{Input Channels:32 ; Output Channels:32 ; Kernel Size:4 ; Stride:2 ; Padding:1 } & ReLU \\  
				& Flatten & Start Dimension:1 &   \\ 
				& Linear & Input Dimension: 512 ; Output Dimension: 256 & ReLU  \\
				& Linear & Input Dimension: 256 ; Output Dimension: 256 & ReLU \\ 
				& Linear & Input Dimension: 256 ; Output Dimension: 12 & ReLU \\ 
				\hline
				Reparametrization Trick & \multicolumn{3}{l}{\makecell[c]{The output of the encoder contains $\vaemu$ and $\log \vaesigma$. \\ The latent representation is then generated via $\h = \vaemu(\x) + \vaesigma(\x) \odot \noise$, $\noise \sim \mathcal{N}(0, I)$}} \\
				\hline
				\multirow{8}*{Decoder} & Linear & Input Dimension: 6 ; Output Dimension: 256 & ReLU  \\
				& Linear & Input Dimension: 256 ; Output Dimension: 256 & ReLU  \\
				& Linear & Input Dimension: 256 ; Output Dimension: 512 & ReLU  \\
				& Unflatten & Dimension:1 ; Unflatten Size:(32, 4, 4) &   \\
				& Conv2d & \makecell[l]{Input Channels:32 ; Output Channels:32 ; Kernel Size:4 ; Stride:2 ; Padding:1 } & ReLU \\  
				& ConvTranspose2d & \makecell[l]{Input Channels:32 ; Output Channels:32 ; Kernel Size:4 ; Stride:2 ; Output Padding:1} &  ReLU \\  
				& ConvTranspose2d & \makecell[l]{Input Channels:32 ; Output Channels:32 ; Kernel Size:4 ; Stride:2 ; Output Padding:1} & ReLU  \\
				& ConvTranspose2d & \makecell[l]{Input Channels:32 ; Output Channels:1 ; Kernel Size:4 ; Stride:2 ; Output Padding:1} & Sigmoid  \\
				\bottomrule
			\end{tabular}
		\end{adjustbox}
	\end{center}
\end{table}

\paragraph{Model.} The architecture of the MNIST VAE is described in Table~\ref{tab:vae_mnist} and those of the dSprites VAE is described in Table~\ref{tab:vae_dsprites}. Both of these architectures are reproductions of the VAEs from \citet{Burgess2018}. The $\beta$-VAE is trained to minimize the objective $L_{\beta}(\x, \theta, \phi) = \E_{q_{\phi}(\h \mid \x)} \left[ \log p_{\theta}(\x \mid \h)\right] - \beta D_{\textrm{KL}} \left[ q_{\phi}(\h \mid \x)  \mid \mid p(\h) \right]$, where $q_{\phi}(\h \mid \x)$ is the distribution underlying the reparametrized encoder output, $p_{\theta}(\x \mid \h)$ is the distribution underlying the decoder output, $p(\h)$ is the density associated to isotropic unit Gaussian $\mathcal{N}(\boldsymbol{0}, \boldsymbol{I})$ underlying $\noise$ and  $D_{\textrm{KL}}$ is the KL-divergence. The objective of the TC-VAE is th same as in~\cite{Chen2018}. We refer the reader to the original paper for the details. All the VAEs  are trained to minimize their objective for 100 epochs with patience 10 by using Pytorch's Adam with hyperparameters: $ \textrm{learning rate} = .001, \beta_1 = .9, \beta_2 = .999, \epsilon = 10^{-8}, \textrm{weight decay} = 10^{-5}$.  

\paragraph{Feature Importance.} As a baseline for the feature importance methods, we use a black image $\bar{\x} = \boldsymbol{0}$.

\paragraph{Metrics.} We use the Pearson coefficient to measure the correlation between two importance scores given a random test example and a random feature. In this case, one correlation coefficient can be computed for each couple $(i,j) \in [d_H]^2$ of latent units:
\begin{align*}
	&r_{ij} = \frac{\cov_{\rx, I} \left[ a_I(\mu_i, \rx) , a_I(\mu_j, \rx) \right]}{\sigma_{\rx, I}\left[ a_I(\mu_i, \rx) \right] \sigma_{\rx, I}\left[ a_I(\mu_j, \rx) \right]} \\
	&\rx \sim \textrm{Empirical Distribution}(\Dtest) , I \sim \textrm{Uniform}([d_X]),
\end{align*}
where $\mu_i$ is the $i$-th component of the expected representation computed by the encoder for all $i \in [d_H]$. To have an overall measure of correlation between the VAE units, we sum over all pairs of distinct latent units:
\begin{align*}
	r = \frac{1}{d_H (d_H -1)} \sum_{i=1 , i\neq j}^{d_H} \sum_{j=1}^{d_H} r_{ij}
\end{align*}
This averaged correlation coefficient is the one that we report in Figure~\ref{fig:vae_pearson}. In our quantitative analysis, we also report the Spearman rank correlation between $\beta$ and $r$. Concretely, this is done by performing the experiment $M \in \N^*$ times for different values $\beta_1 , \dots , \beta_M$ of $\beta$. We then measure the correlation coefficients $r_1 , \dots , r_M$ associated to each experiment. The Spearman rank correlation coefficient can be computed form this data:
\begin{align*}
	&\rho = \frac{\cov_{M_1 , M_2} \left[ \rank(r_{M_1}) , \rank(\beta_{M_2})\right]}{\sigma_{M_1}[\rank(r_{M_1})] \sigma_{M_2}[\rank(\beta_{M_2})]}\\
	& M_1 , M_2 \sim \textrm{Uniform}([M]).
\end{align*}
This coefficient $\rho$ ranges from $-1$ to $1$, where $\rho = -1$ corresponds to a perfect monotonically decreasing relation, $\rho = 0$ corresponds to the absence of monotonic relation and $\rho = 1$ corresponds to a perfect monotonically increasing relation.

\paragraph{Entropy.} To complete our quantitative analysis of the VAEs, we introduce a new metric called \emph{entropy}. The purpose of this metric is to measure how the saliency for each feature is distributed across the different latent units. In particular, we would like to be able to distinguish the case where all latent units are sensitive to a feature and the case where only one latent unit is sensitive to a feature. As we have done previously, we can compute the importance score $a_i(\mu_j, \x)$ of each feature $x_i$ from $\x \in \X$ for a latent unit $j \in [d_H]$. For each latent unit $j \in [d_H]$, we define the proportion of attribution as
\begin{align*}
	p_{j}(i, \x) = \frac{|a_i(\mu_j , \x)|}{\sum_{k=1}^{d_H} |a_i(\mu_k , \x)|}.
\end{align*}
This corresponds to the fraction of the importance score attributed to $j$ for feature $i$ and example $\x$. Note that this quantity is well defined if at least one of the $a_i(\mu_k , \x), k \in [d_H]$ is non-vanishing. Hence, we only consider the features $i \in [d_X]$ that are salient for at least one latent unit. We can easily check that $\sum_{j=1}^{d_H} p_{j}(i, \x) = 1$ by construction. This means that the proportions of attributions can be interpreted as probabilities of saliency. This allows us to define an entropy that summarizes the distribution over the latent units:
\begin{align*}
	S(i, \x) = - \sum_{j=1}^{d_H} p_{j}(i, \x) \ln p_{j}(i, \x).
\end{align*}
This entropy is analogous to Shannon's entropy~\cite{Shannon1948}. It can be checked easily that this entropy is minimal ($S^{\textrm{min}} = 0$) whenever only one latent unit $j \in [d_H]$ is sensitive to feature $i$: $p_{j}(i, \x) = 1$. Conversely, it is well known~\cite{Cover2005} that the entropy is maximal ($S^{\textrm{max}} = \ln d_H$) whenever the distribution is uniform over the latent units: $p_{j}(i, \x) = \nicefrac{1}{d_H}$ for all $j \in [d_H]$. In short: the entropy is low when mostly one latent unit is sensitive to the feature of interest and high when several latent units are sensitive to the feature of interest. Clearly, the former situation is more desirable if we want to distinguish the different latent units. For each VAE, we evaluate the average entropy
\begin{align*}
	&S = \E_{\rx , I} \left[ S(I, \rx) \right]\\
	&\rx \sim \textrm{Empirical Distribution}(\Dtest) , I \sim \textrm{Uniform}[\mathcal{I}(\rx)], 
\end{align*}  
where $\mathcal{I}(\x) = \{ i \in [d_X] \mid a_i(\mu_k , \x) \neq 0 \textrm{ for at least one } k \in [d_H]\}$ is the set of features that are salient for at least one latent unit. We measure the average entropy for each VAE and report the results as a function of $\beta$ in Figure~\ref{fig:vae_entropy}.

\begin{figure*}[h] 
	\centering
	\subfigure[MNIST]{
		\begin{minipage}[t]{0.5\linewidth}
			\centering
			\includegraphics[width=\linewidth]{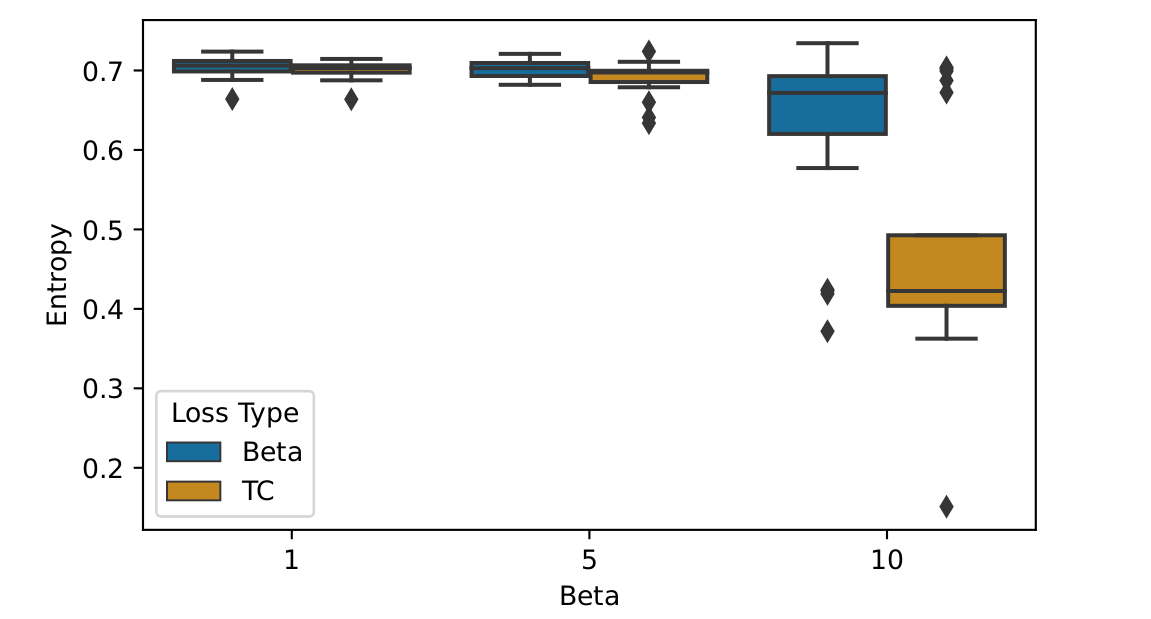}
		\end{minipage}
	}%
	\subfigure[dSprites]{
		\begin{minipage}[t]{0.5\linewidth}
			\centering
			\includegraphics[width=\linewidth]{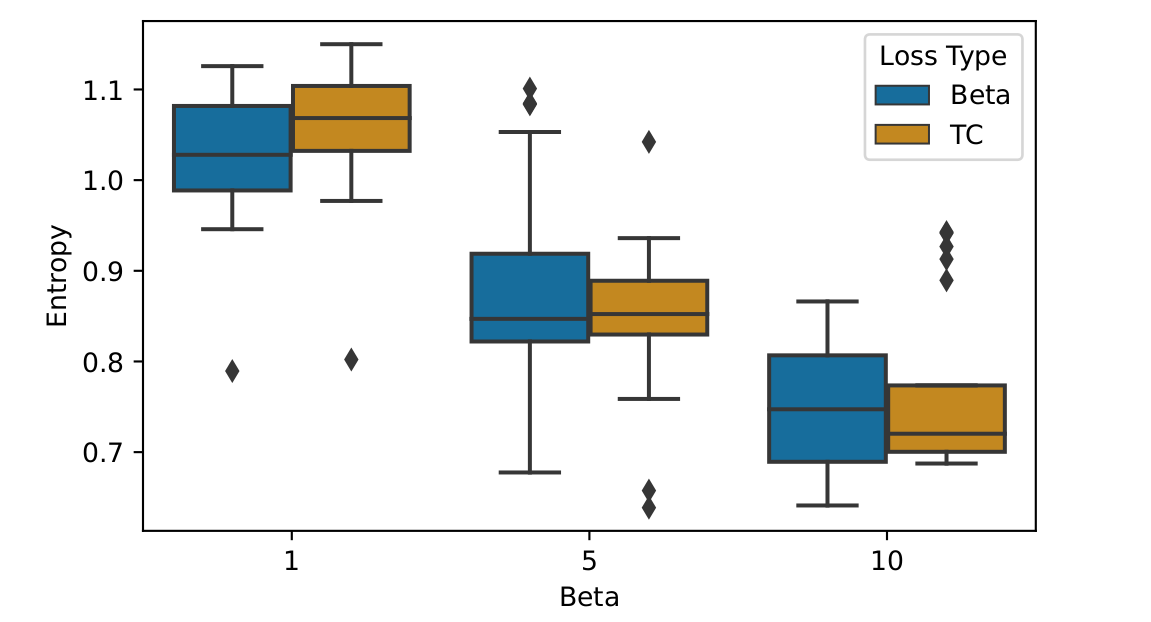}
		\end{minipage}
	}
	\caption{Entropy of saliency maps for different values of $\beta$.}
	\label{fig:vae_entropy}	
\end{figure*}

We clearly see that the entropy decreases as the disentanglement strength $\beta$ increases for both MNIST (Spearman $\rho = -.56$) an dSprites (Spearman $\rho = -.76$). This means that disentangling has the effect of distributing the saliency over fewer units. This brings a nice complement to the quantitative analysis that we have conducted in Section~\ref{subsec:vaes}: although increasing disentanglement does not make the latent units focus on different parts of the image (since the correlation does not decrease significantly), it does decrease the number of latent units that are simultaneously sensitive to a given part of the image (since the entropy decreases substantially). These two phenomena are not incompatible with each other. For instance, we see that the 6-th latent unit seems inactive in comparison with the other latent units in Figure~\ref{fig:vae_disprites_appendix}. In fact, this latent unit might perfectly pay attention to the same parts of the image as the other units and, hence, be correlated. What distinguishes this unit from the others is that the feature importance scores on its saliency map are significantly smaller (we cannot appreciate it by plotting the saliency maps on the same scale) and, hence, reduces the entropy. Finally, we note that the entropies from Figure~\ref{fig:vae_entropy} remain fairly close to their maximal value ($S^{max} = \ln 3 \approx 1.1$ for MNIST and $S^{max} = \ln 6 \approx 1.8$ for dSprites). This means that the VAEs have several active units for each pixel.

\paragraph{Supplementary Examples.} To check that the qualitative analysis from Section~\ref{subsec:vaes} extends beyond the examples showed in the main paper, the reader can refer to Figures~\ref{fig:vae_mnist_appendix}~and~\ref{fig:vae_disprites_appendix}. These saliency maps are produced with the main paper's VAEs. We also plot saliency maps for vanilla ($\beta = 1$) VAEs in Figures~\ref{fig:vanilla_vae_mnist}~and~\ref{fig:vanilla_vae_dsprites}. The issues mentioned in Section~\ref{subsec:vaes} are still present in this case.

\FloatBarrier
\begin{figure} 
	\centering
	\includegraphics[height=\textheight]{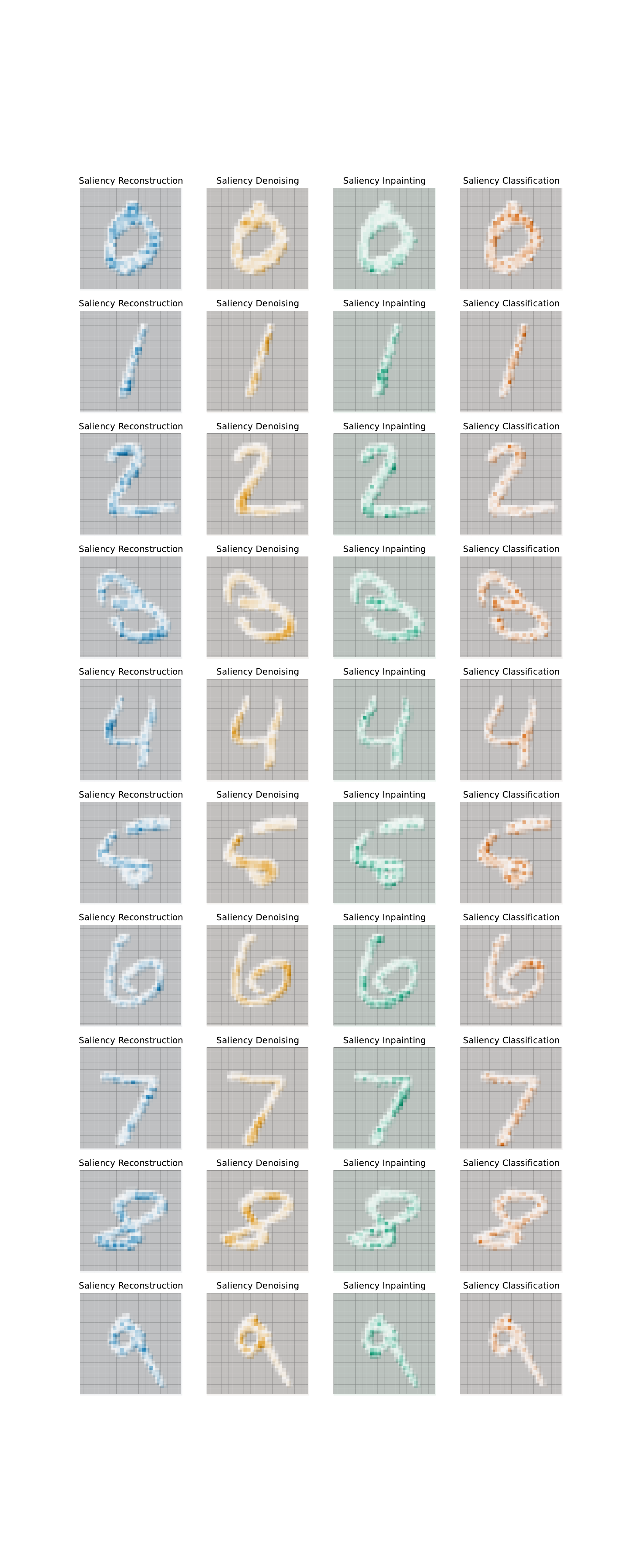}
	\vspace{-.5in}
	\caption{Label-free saliency for various pretext tasks.}
	\label{fig:pretext_saliency_appendix}
\end{figure}

\begin{figure} 
	\centering
	\includegraphics[height=\textheight]{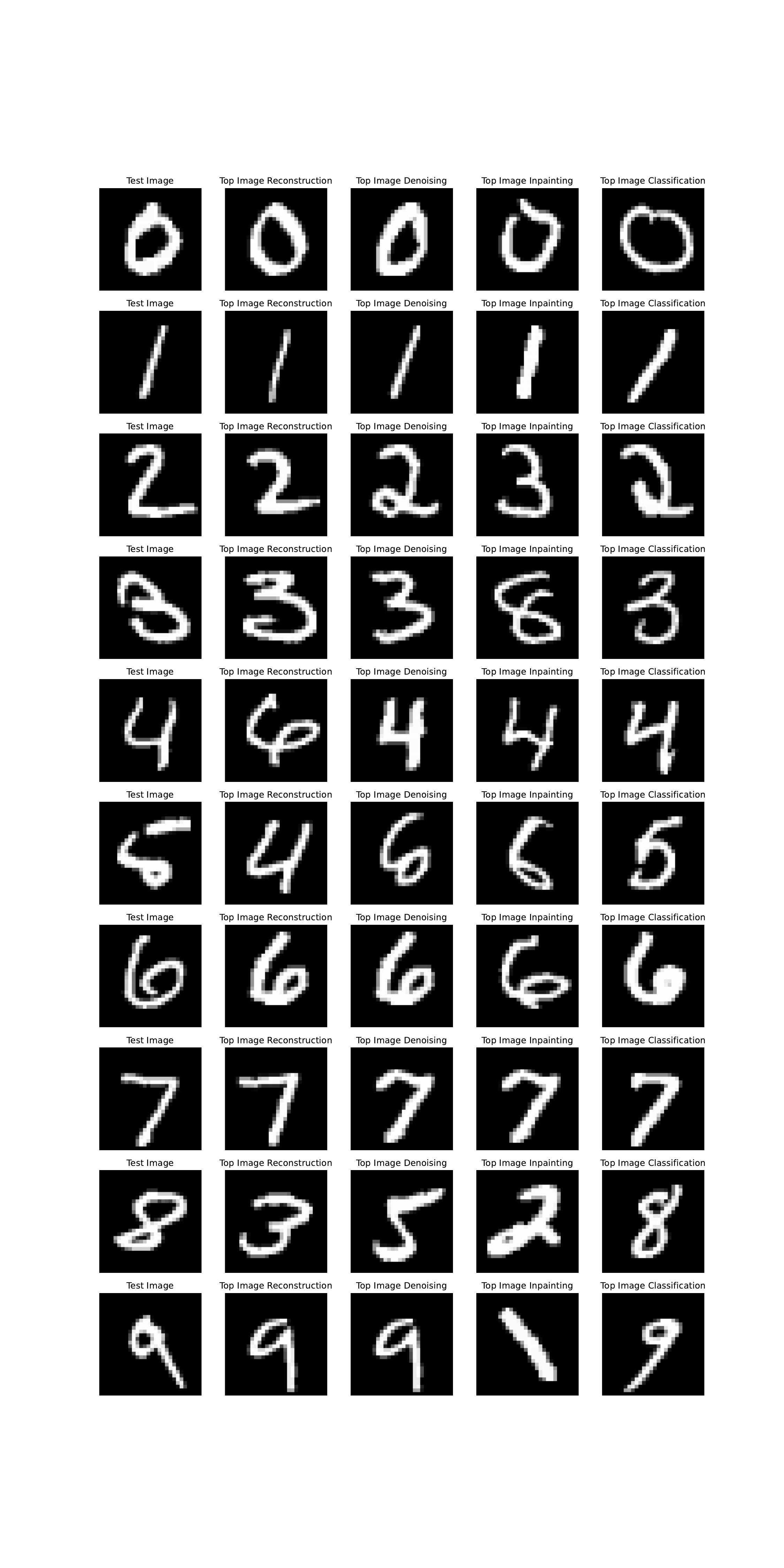}
	\vspace{-.5in}
	\caption{Label-free top example for various pretext tasks.}
	\label{fig:pretext_example_appendix}
\end{figure}

\begin{figure} 
	\centering
	\includegraphics[height=\textheight]{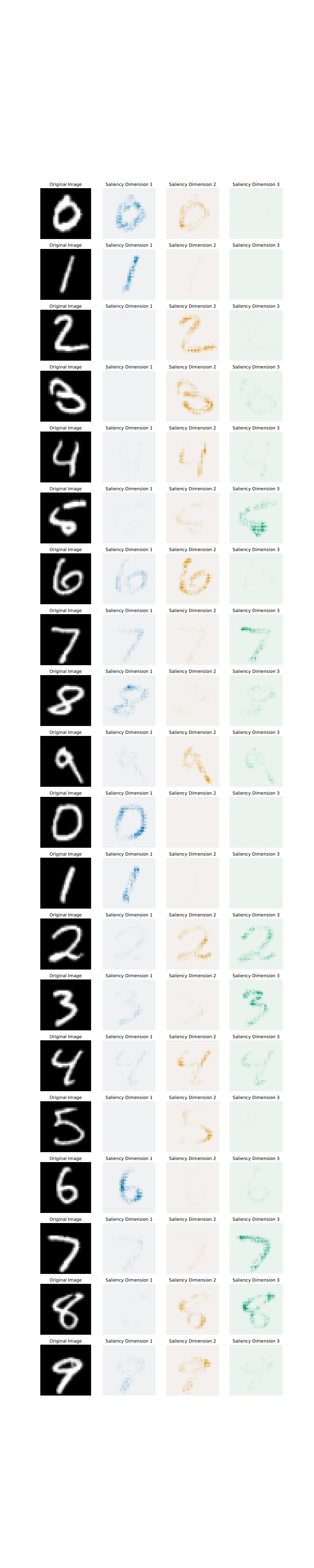}
	\vspace{-.5in}
	\caption{Saliency maps for the latent units of a MNIST VAE.}
	\label{fig:vae_mnist_appendix}
\end{figure}

\begin{figure} 
	\centering
	\includegraphics[height=\textheight]{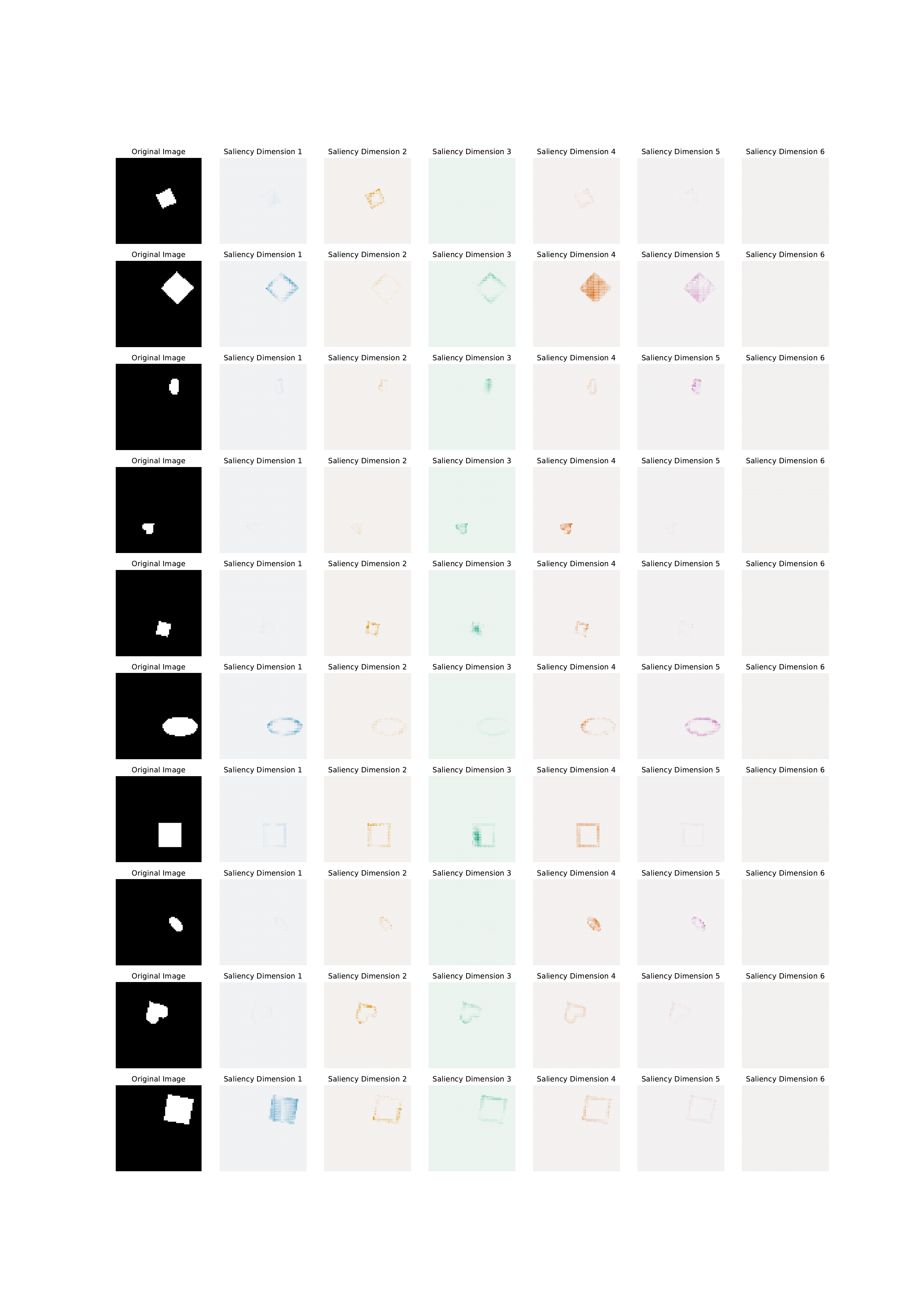}
	\vspace{-.5in}
	\caption{Saliency maps for the latent units of a dSprites VAE.}
	\label{fig:vae_disprites_appendix}
\end{figure}

\begin{figure} 
	\centering
	\includegraphics[height=\textheight]{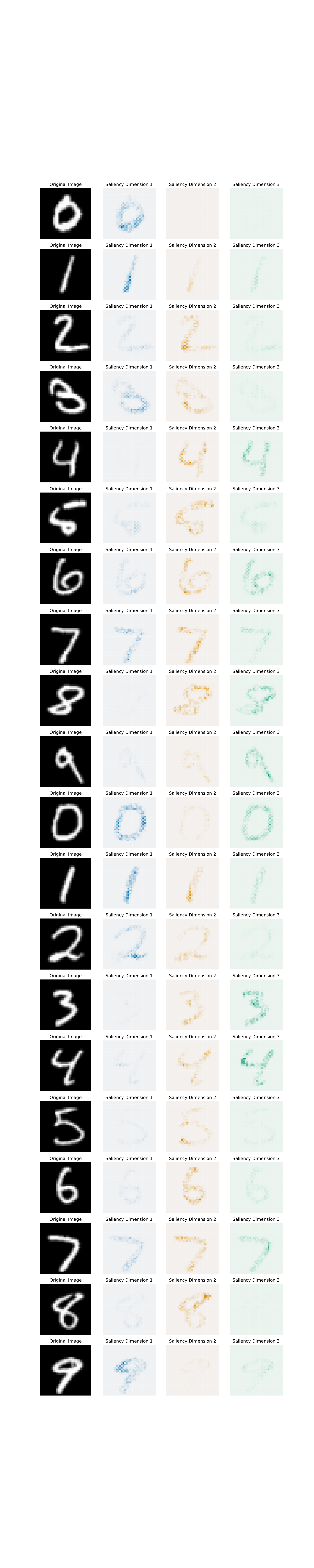}
	\vspace{-.5in}
	\caption{Saliency maps for the latent units of a vanilla ($\beta = 1$) MNIST VAE.}
	\label{fig:vanilla_vae_mnist}
\end{figure}

\begin{figure} 
	\centering
	\includegraphics[height=\textheight]{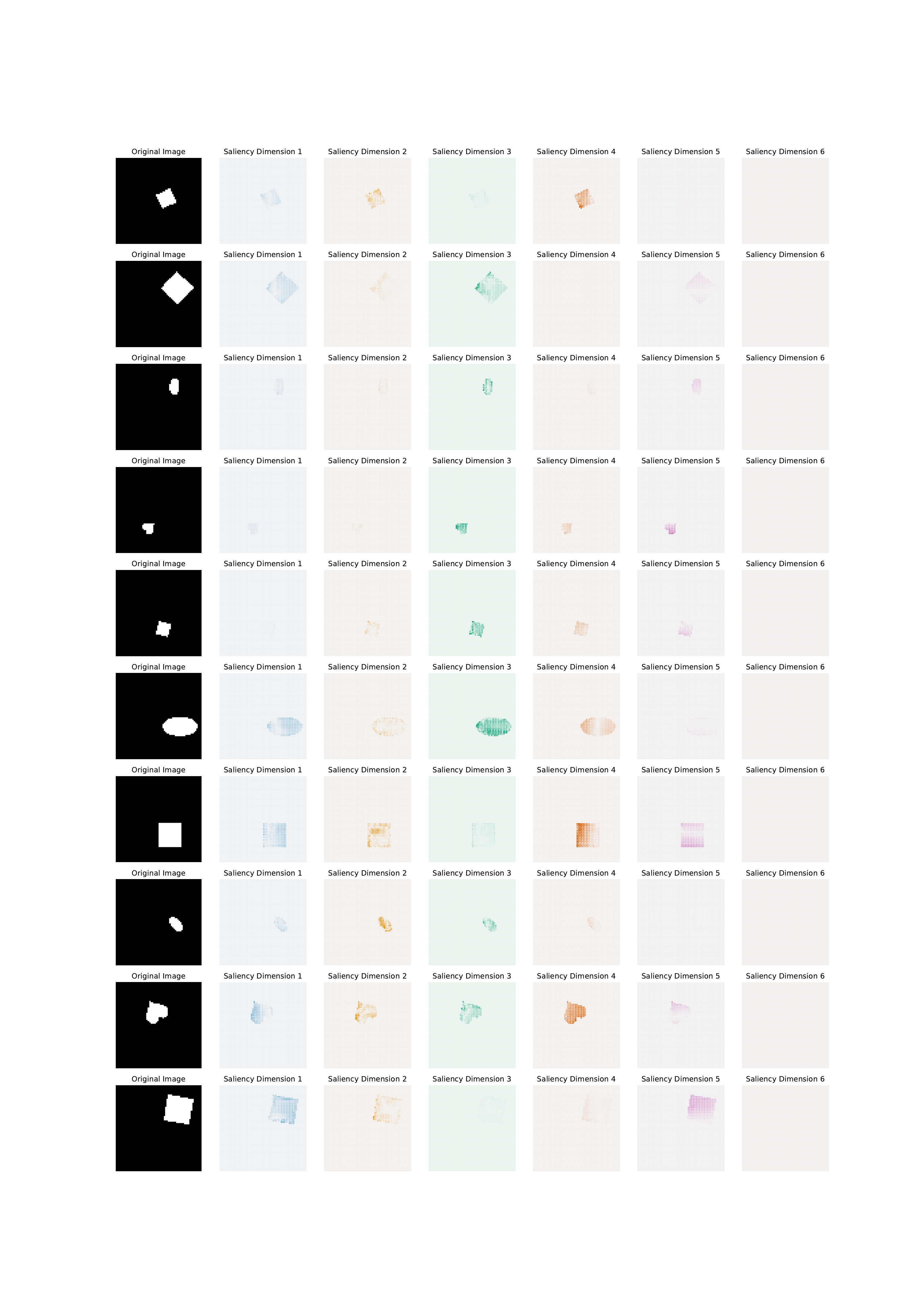}
	\vspace{-.5in}
	\caption{Saliency maps for the latent units of a vanilla ($\beta = 1$) dSprites VAE.}
	\label{fig:vanilla_vae_dsprites}
\end{figure}

\end{document}